\pdfoutput=1
\documentclass{article}

\usepackage{microtype}
\usepackage{graphicx}
\usepackage{subfigure}
\usepackage{booktabs} 

\usepackage{hyperref}

\usepackage[accepted]{icml2018}

\usepackage{amsmath}
\usepackage{amssymb}
\newtheorem{theorem}{Theorem}

\newenvironment{proof}{{\noindent\it Proof}\quad}{\hfill $\square$\par}
\usepackage{url}

\def\g{{\bf g}}

\def\u{{\bf u}}
\def\v{{\bf v}}
\def\w{{\bf w}}
\def\x{{\bf x}}

\def\z{{\bf z}}

\def\D{{\bf D}}
\def\E{{\bf E}}

\def\0{{\bf 0}}
\def\1{{\bf 1}}
\def\2{{\bf 2}}
\def\3{{\bf 3}}
\def\4{{\bf 4}}
\def\5{{\bf 5}}
\def\6{{\bf 6}}
\def\7{{\bf 7}}
\def\8{{\bf 8}}
\def\9{{\bf 9}}

\def\RB{{\mathbb R}}

\begin{document}

\twocolumn[
\icmltitle{Feature-Distributed SVRG for High-Dimensional Linear Classification}



\icmlsetsymbol{equal}{*}

\begin{icmlauthorlist}
\icmlauthor{Gong-Duo Zhang}{nju}
\icmlauthor{Shen-Yi Zhao}{nju}
\icmlauthor{Hao Gao}{nju}
\icmlauthor{Wu-Jun Li}{nju}
\end{icmlauthorlist}

\icmlaffiliation{nju}{Nanjing University, Nanjing, China}

\icmlcorrespondingauthor{Gong-Duo Zhang}{zhanggd@lamda.nju.edu.cn}
\icmlcorrespondingauthor{Shen-Yi Zhao}{zhaosy@lamda.nju.edu.cn}
\icmlcorrespondingauthor{Hao Gao}{gaoh@lamda.nju.edu.cn}
\icmlcorrespondingauthor{Wu-Jun Li}{liwujun@nju.edu.cn}

\icmlkeywords{Machine Learning, ICML}

\vskip 0.3in
]



\printAffiliationsAndNotice{}  

\begin{abstract}

Linear classification has been widely used in many high-dimensional applications like text classification. To perform linear classification for large-scale tasks, we often need to design distributed learning methods on a cluster of multiple machines. In this paper, we propose a new distributed learning method, called feature-distributed stochastic variance reduced gradient~(\mbox{FD-SVRG}) for high-dimensional linear classification. Unlike most existing distributed learning methods which are instance-distributed, \mbox{FD-SVRG} is feature-distributed. \mbox{FD-SVRG} has lower communication cost than other instance-distributed methods when the data dimensionality is larger than the number of data instances. Experimental results on real data demonstrate that \mbox{FD-SVRG} can outperform other state-of-the-art distributed methods for high-dimensional linear classification in terms of both communication cost and wall-clock time, when the dimensionality is larger than the number of instances in training data.
\end{abstract}

\section{Introduction}\label{sec:intro}

Linear classification models, such as logistic regression~(\mbox{LR}) and linear support vector machine~(SVM), can achieve good performance in many high-dimensional applications like text classification. When the training set is too large to be handled by one single machine~(node), we often need to design distributed learning methods on a cluster of multiple machines. Hence, it has become an interesting topic to design distributed linear classification models for some large-scale tasks with high-dimensionality.

For large-scale linear classification problems, stochastic gradient descent~(SGD) and its variants like stochastic average gradient~(SAG)~\cite{DBLP:journals/mp/SchmidtRB17}, SAGA~\cite{DBLP:conf/nips/DefazioBL14}, stochastic dual coordinate ascent~(SDCA)~\cite{DBLP:journals/jmlr/Shalev-Shwartz013,DBLP:conf/icml/Shalev-Shwartz014} and stochastic variance reduced gradient~(SVRG)~\cite{DBLP:conf/nips/Johnson013} have shown promising performance in real applications. Hence, most existing distributed learning methods adopt SGD or its variants for updating~(learning) the parameters of the linear classification models. Representatives include PSGD~\cite{DBLP:conf/nips/ZinkevichWSL10}, DC-ASGD~\cite{DBLP:conf/icml/ZhengMWCYML17}, DisDCA~\cite{DBLP:conf/nips/Yang13}, CoCoA~\cite{DBLP:conf/nips/JaggiSTTKHJ14}, CoCoA+~\cite{DBLP:conf/icml/MaSJJRT15} and DSVRG~\cite{lee2015distributed}.

According to the organization framework of the cluster, existing distributed learning methods can be divided into three main categories. The first category is based on the master-slave framework, which has one master node~(machine) and some slave nodes. In general, the model parameter is stored in the master node, and the data is distributively stored in the slave nodes. The master node is responsible for updating the model parameter, and the slave nodes are responsible for computing the gradient or stochastic gradient. This category is a centralized framework. One representative of this category is MLlib on Spark \cite{DBLP:journals/jmlr/MengBYSVLFTAOXX16}.  The bottleneck of this kind of centralized framework is the high communication cost on the central~(master) node~\cite{DBLP:conf/nips/LianZZHZL17}. The second category is based on the Parameter Server framework~\cite{DBLP:conf/osdi/LiAPSAJLSS14,DBLP:conf/nips/LiASY14,DBLP:conf/kdd/XingHDKWLZXKY15}, which has two kinds of nodes called Servers and Workers respectively. The Servers are used to store and update the model parameter, and the Workers are used to distributively store the data and compute the gradient or stochastic gradient. PS-Lite\footnote{PS-Lite is called Parameter Server in the original paper~\cite{DBLP:conf/osdi/LiAPSAJLSS14}. In this paper, we use Parameter Server to denote the general framework, and use PS-Lite for the specific Parameter Server platform in~\cite{DBLP:conf/osdi/LiAPSAJLSS14}. PS-Lite can be downloaded from \url{https://github.com/dmlc/ps-lite}.}~\cite{DBLP:conf/osdi/LiAPSAJLSS14} and Petuum \cite{DBLP:conf/kdd/XingHDKWLZXKY15} are two representatives of this category. Parameter Server is also a centralized framework. Unlike the master-slave framework which has only one master node for model parameter, Parameter Server can use multiple Servers to distributively store and update the model parameter, and hence can relief the communication burden on the central nodes~(Servers). However, there also exists frequent communication of gradients and parameters between Servers and Workers, because the parameter and data are stored separately in Servers and Workers in Parameter Server framework. The third category is the decentralized framework, in which there are only workers and no central nodes~(servers). The data is distributively stored on all the workers, and all workers need to store and update~(learn) the model parameter. D-PSGD~\cite{DBLP:conf/nips/LianZZHZL17} and DSVRG~\cite{lee2015distributed} are two representatives of this category. D-PSGD is a distributed SGD method that abandons the central node and need much less communication on the busiest node compared to centralized frameworks. But it still need to communicate parameter vector frequently between workers. DSVRG is a distributed SVRG method which has a ring framework. Because the convergence rate of SVRG is much faster than SGD~\cite{DBLP:conf/nips/Johnson013}, DSVRG also converges much faster than other SGD-based distributed methods. Furthermore, the decentralized framework of DSVRG also avoids the communication bottleneck in the centralized frameworks. Hence, DSVRG has achieved promising performance for learning linear classification models.

Most existing distributed learning methods, including all the centralized and decentralized methods mentioned above, are instance-distributed, which partition the training data by instances. These instance-distributed methods have achieved promising performance in large-scale problems when the number of instances is larger than dimensionality~(the number of features). In some real applications like web mining, astronomical projects, financial and biotechnology applications, the dimensionality can be larger than the number of instances~\cite{negahban2012unified}. For these cases, the communication cost of instance-distributed methods is typically high, because they often need to communicate the high-dimensional parameter vectors or gradient vectors among different machines.

In this paper, we propose a new distributed SVRG method, called feature-distributed SVRG~(\mbox{FD-SVRG}), for high-dimensional linear classification. The contributions of \mbox{FD-SVRG} are briefly listed as follows:
\begin{itemize}
\item  Unlike most existing distributed learning methods which are instance-distributed, \mbox{FD-SVRG} is feature-distributed.
\item FD-SVRG has the same convergence rate as the non-distributed~(serial) SVRG, while the parameters in \mbox{FD-SVRG} can be distributively learned.
\item FD-SVRG has lower communication cost than other instance-distributed methods when the dimensionality is larger than the number of instances.
\item Experimental results on real data demonstrate that \mbox{FD-SVRG} can outperform other state-of-the-art distributed SVRG methods in terms of both communication cost and wall-clock time, when the dimensionality is larger than the number of instances in training data.
\item In particular, compared with the Parameter Server \mbox{PS-Lite} which has been widely used by both academy and industry, FD-SVRG is several orders of magnitude faster.
\end{itemize}

Please note that our feature-distributed framework is not only applicable to SVRG, it can also be applied to SGD and other variants. Furthermore, it can also be used for regression or other liner models. Due to space limitation, we only focus on SVRG based linear classification here and leave other variants for further study.

\section{Problem Formulation}
Although there exist different formulations for the linear classification problem, this paper focuses the most popular formulation shown as follows:
\begin{align}\label{eq:objectiveF}
  &\min_\textbf{w}  f(\w) = \frac{1}{N} \sum_{i=1}^N f_{i}(\textbf{w}), \\
  &f_{i}(\textbf{w}) = \varphi_{i}(\textbf{w}^T\textbf{x}_i,y_i)+  g(\textbf{w}),
\end{align}
where $N$ is the number of instances in the training set, $\x_i \in \RB^d$ is the feature vector of instance $i$, $y_i \in \{-1,+1\}$ is the class label of instance $i$, $d$ is the dimensionality~(number of features) of the instances, $\w \in \RB^d$ is the parameter to learn, $\varphi_{i}(\textbf{w}^T\textbf{x}_i,y_i)$ is the loss defined on instance $i$, $g(\w)$ is a regularization function. Here, we only focus on two-class problems, but the techniques in this paper can also be adapted for multi-class problems which are omitted for space saving.

Many popular linear classification models can be formulated as the form in~(\ref{eq:objectiveF}). For example, in logistic regression~(LR), $ f_i(\textbf{w}) = \log (1 + e^{-y_i\textbf{w}^T\textbf{x}_i}) + \frac{\lambda}{2} ||\textbf{w}||_2^2$, where $\varphi_{i}(\textbf{w}^T\textbf{x}_i,y_i)$ is the logistic loss $\log (1 + e^{-y_i\textbf{w}^T\textbf{x}_i})$ and $g(\w) =\frac{\lambda}{2} ||\textbf{w}||_2^2$ is the $L_2$-norm regularization function with a hyper-parameter $\lambda$. In linear SVM, $ f_i(\textbf{w}) = \max\{0,1 - y_i \textbf{w}^T \textbf{x}_i\}+\frac{\lambda}{2} ||\textbf{w}||_2^2$.

In many real applications, the training set can be too large to be handled by one single machine~(node). Hence, we need to design distributed learning methods to learn~(optimize) the parameter $\w$ based on a cluster of multiple machines~(nodes). In some applications, $N$ can be larger than $d$. And in other applications, $d$ can be larger than $N$. In this paper, we focus on the case when $d$ is larger than $N$, which has attracted much attention in recent years~\cite{DBLP:journals/focm/ChandrasekaranRPW12,negahban2012unified,DBLP:conf/icml/WangDL16,DBLP:conf/icml/SivakumarB17}.

\section{Related Work}\label{sec:relatedWork}
As stated in Section~\ref{sec:intro}, there exist three main categories of distributed learning methods for the problem in~(\ref{eq:objectiveF}). Here, we briefly introduce these existing methods to motivate the contribution of this paper. Because the master-slave framework can be seen as a special case of Parameter Server with only one Server, here we only introduce Parameter Server and the decentralized framework.

\subsection{Parameter Server}
The Parameter Server framework~\cite{DBLP:conf/osdi/LiAPSAJLSS14,DBLP:conf/kdd/XingHDKWLZXKY15} is illustrated in Figure~\ref{PSframework}, in which there are $p$ Servers and $q$ Workers. Let $\textbf{D}\in \mathbb{R}^{d\times N}$ denote the training data matrix, where the $i$th column of $\D$ denotes the $i$th instance $\x_i$. In Parameter Server, the whole training set $\D$ is instance-distributed, which means that $\D$ is partitioned vertically (by instance) into $q$ subsets $\{\textbf{D}_1,\textbf{D}_2, \cdots, \textbf{D}_q\}$ and $\D_k$ will be assigned to Worker$\_k$. The parameter $\textbf{w} \in \mathbb{R}^d$ is cut off into $p$ parts $\{\textbf{w}^{(1)}, \textbf{w}^{(2)}, \cdots,\textbf{w}^{(p)}\}$ and $\textbf{w}^{(k)}$ will be assigned to Server$\_k$. Servers are responsible for updating parameters, and  Workers are responsible for computing gradients.
\begin{figure}[t]
  \centering
  \includegraphics[width=0.29\textwidth]{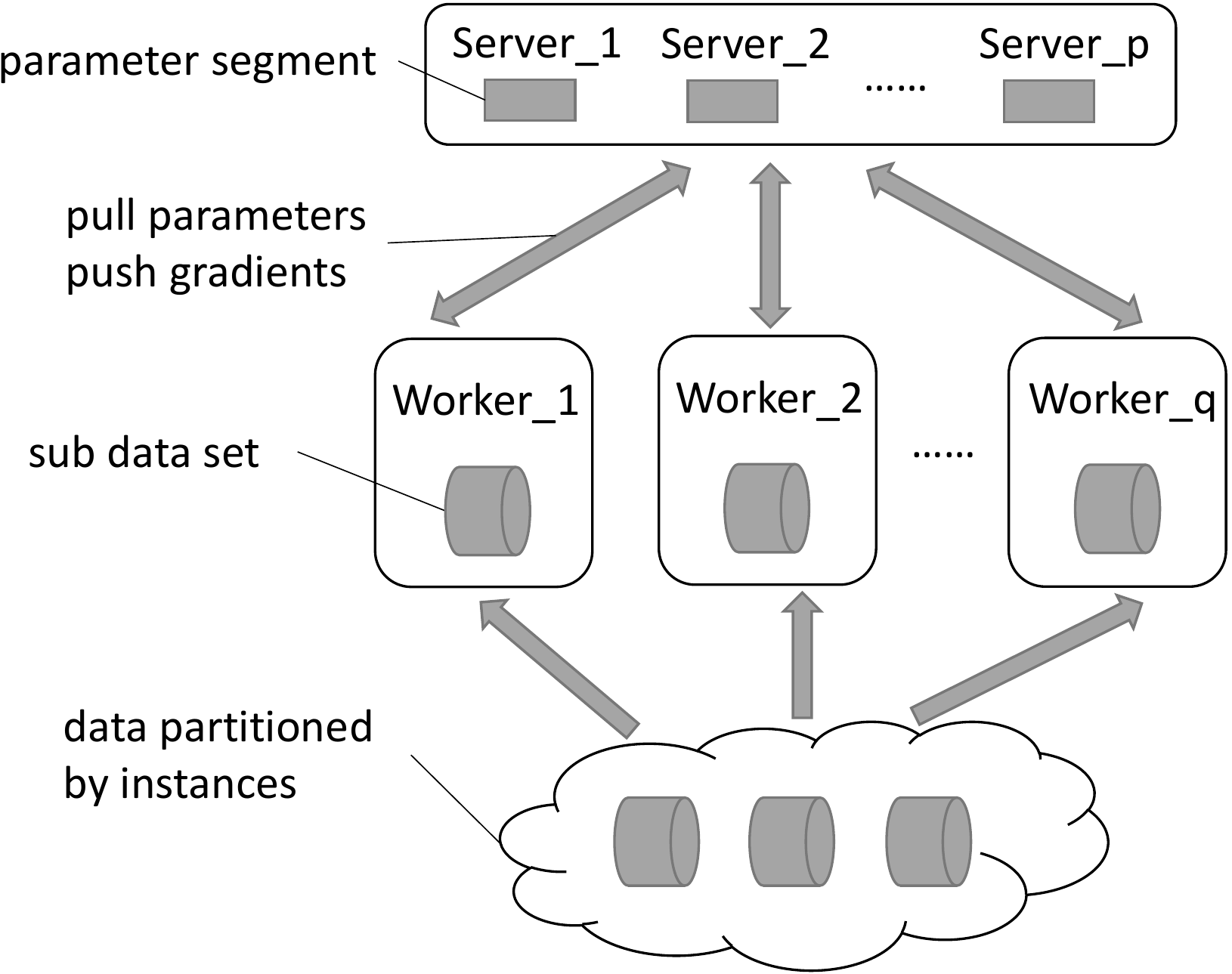}\\
  \caption{ Parameter Server framework.}\label{PSframework}
\end{figure}
The communication is done by \textit{pull} and \textit{push} operations. That is to say, the Workers will pull parameters from Servers, and push gradients to Servers. When the number of Workers becomes larger, there are more machines for gradient computation at the same time and the \textit{pull} and \textit{push} requests will be more frequent for Servers. Although PS-Lite and Petuum can use $\langle key, value\rangle$ data structure to decrease the communication cost for sparse data, the total communication cost is still high because of the frequent communication. Adding more Servers can make the parameter segment $\textbf{w}^{(k)}$ in each Server become smaller but cannot reduce the number of communication requests for Servers. Furthermore, as the number of Servers increases, each Worker need to cut the parameter into more segments and communicate more frequently. For the SVRG-based methods, such as mini-batch based synchronous distributed SVRG~(SynSVRG) and asynchronous SVRG~(AsySVRG)~\cite{DBLP:conf/nips/ReddiHSPS15,DBLP:conf/aaai/ZhaoL16}\footnote{Many existing AsySVRG methods, such as those in~\cite{DBLP:conf/nips/ReddiHSPS15,DBLP:conf/aaai/ZhaoL16}, are initially proposed for multi-thread system with a shared memory. But these methods can be easily extended for a cluster of multiple machines to get a distributed version. The distributed SVRG, including SynSVRG and AsySVRG, implemented with parameter server can be found in Appendix B of the supplementary material.}, we need to communicate a dense full gradient vector in each epoch and hence the communication is also high.

%
%

\subsection{Decentralized Framework}
Parameter Server is a centralized framework, where the central nodes~(Servers) will become the busiest ones to handle high communication burden from Workers, especially when the number of Workers is large. Recently, decentralized framework~\cite{DBLP:conf/nips/LianZZHZL17,lee2015distributed} is proposed to avoid the communication traffic jam on the busiest machines. This framework is illustrated in Figure~\ref{DCframework}. There are no Servers in the decentralized methods. Each machine~(Worker) will be assigned a subset of the training instances. Based on the local subset of training instances, each Worker computes gradient and update its local parameter. Then different machines communicate parameter among each other. EXTRA~\cite{DBLP:journals/siamjo/ShiLWY15}, \mbox{D-PSGD}~\cite{DBLP:conf/nips/LianZZHZL17} and DSVRG~\cite{lee2015distributed} are the representatives of this kind of decentralized methods. Although the communication cost of decentralized methods is balanced among Workers, the decentralized methods need to communicate dense parameter vectors, even if the data is sparse. When the data is high-dimensional, the communication cost is also very high.

\begin{figure}[t]
  \centering
  \includegraphics[width=0.3\textwidth]{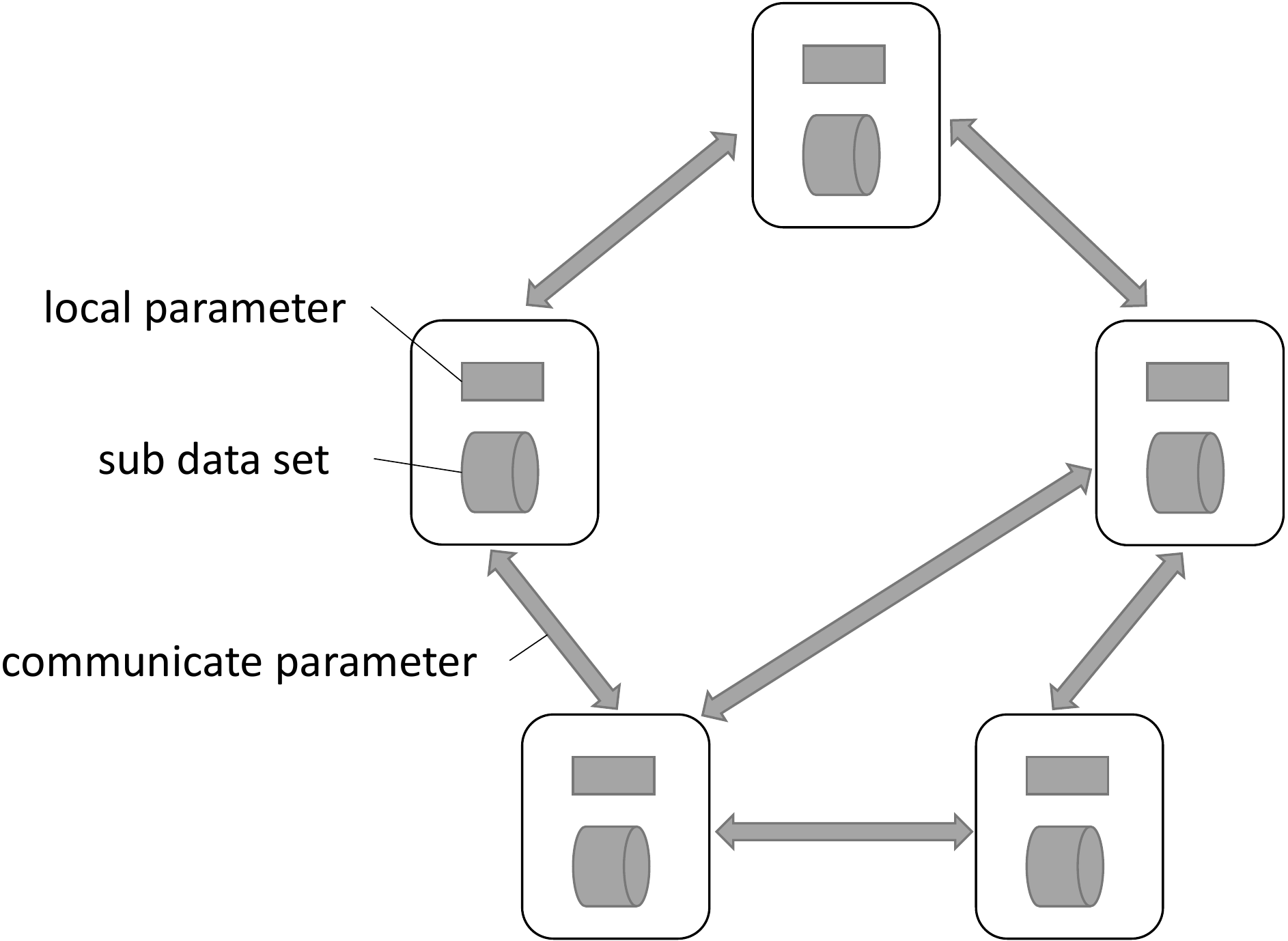}\\
  \caption{Decentralized framework.}\label{DCframework}
\end{figure}

\section{Feature-Distributed SVRG}

Most existing distributed learning methods, including all the centralized and decentralized methods introduced in Section~\ref{sec:relatedWork}, are instance-distributed, which partition the training data by instances. When the dimensionality is larger than the number of instances, i.e., $d>N$, the communication cost of instance-distributed methods is typically high, because they often need to communicate the high-dimensional parameter vectors or gradient vectors of length $d$ among different machines.

In this section, we present our new method called FD-SVRG for high-dimensional linear classification. FD-SVRG is feature distributed, which partitions the training data by features. For training data matrix $\textbf{D}\in \mathbb{R}^{d\times N}$, the difference between instance-distributed partition and feature-distributed partition is illustrated in Figure~\ref{figurepartition}, where the upper-right is feature-distributed partition and the lower-right is instance-distributed partition.

FD-SVRG is based on SVRG, which is much faster than SGD~\cite{DBLP:conf/nips/Johnson013}. For ease of understanding, we briefly present the original non-distributed~(serial) SVRG in Appendix A of the supplementary material.
\begin{figure}[t]
  \centering
  \includegraphics[width=0.21\textwidth]{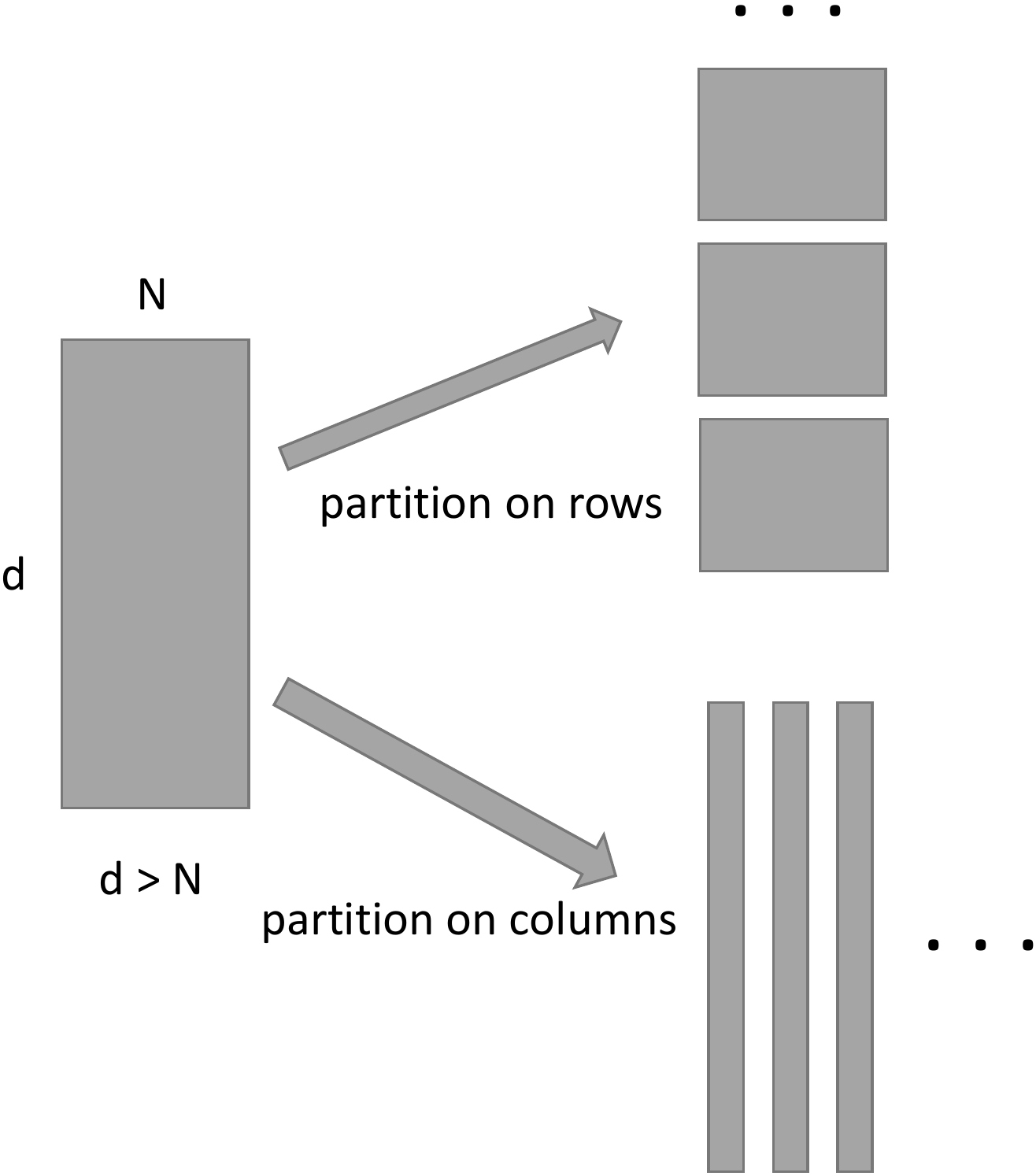}\\
  \caption{The difference between instance-distributed and feature-distributed.}\label{figurepartition}
\end{figure}
\begin{figure}[t]
  \centering
  \includegraphics[width=0.31\textwidth]{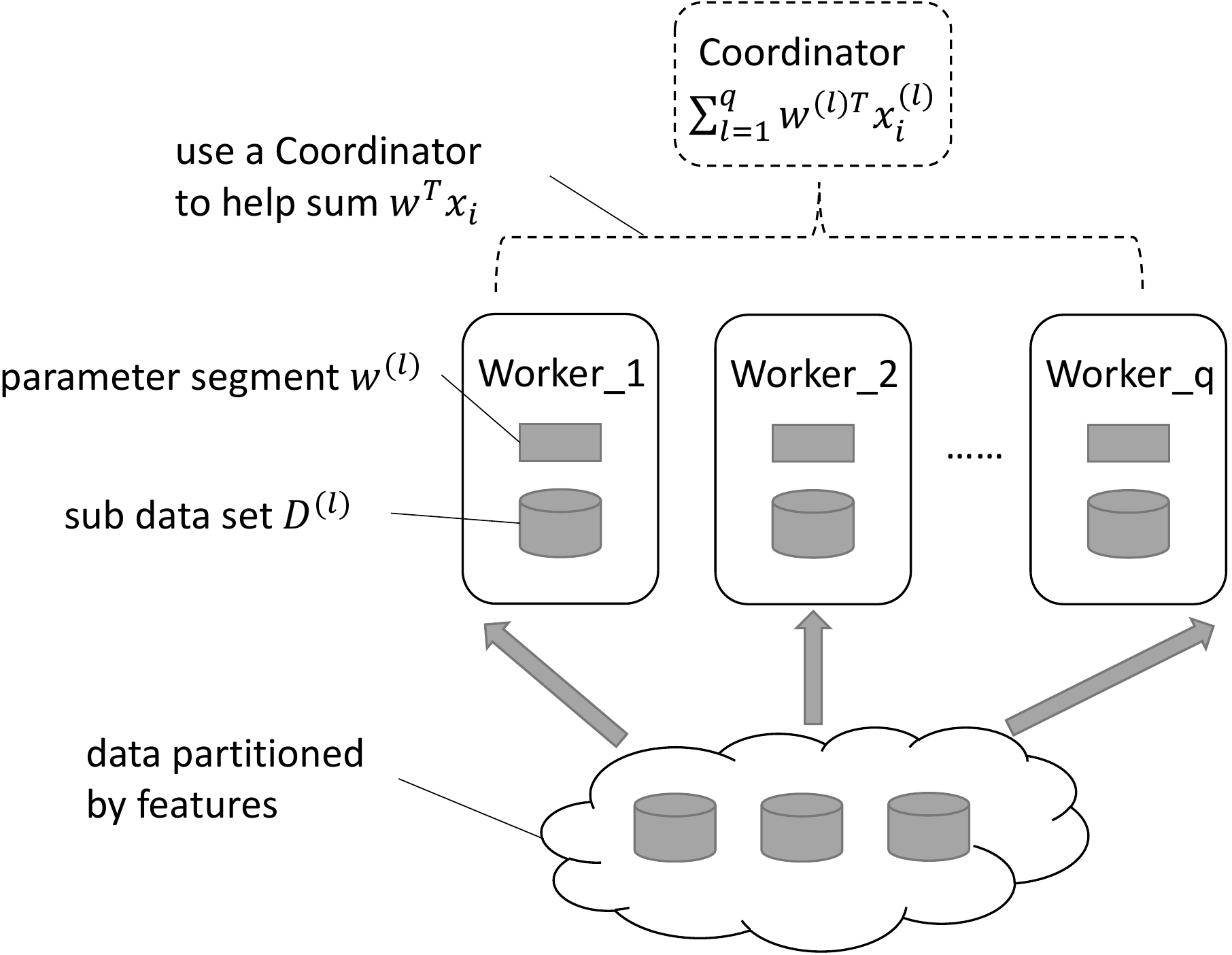}\\
  \caption{Framework of FD-SVRG.}\label{figureLCF}
\end{figure}

\subsection{Framework}

Figure~\ref{figureLCF} shows the distributed framework of FD-SVRG, in which there are $q$ workers and one coordinator. FD-SVRG is feature distributed. More specifically, the training data matrix $\D \in \RB^{d\times N}$ is partitioned horizontally~(by features) into $q$ parts $\D = (\D^{(1)}, \D^{(2)}, $ $ \ldots, \D^{(q)})$, and $\D^{(l)} \in \RB^{d_l\times N}$ is stored on Worker\_$l$. Here, $\sum_{l=1}^q d_l = d$. The parameter $\w$ is also partitioned into $q$ parts $\w=(\w^{(1)},\w^{(2)},\ldots,\w^{(q)})$, with $\w^{(l)}\in \RB^{d_l}$. The features of $\w^{(l)}$ correspond to the features of $\D^{(l)}$. $\w^{(l)}$ is also stored on Worker\_$l$.

\subsection{Learning Algorithm}

First, we rewrite $f_i(\w)$ in~(\ref{eq:objectiveF}) as follows:
\begin{equation}
f_{i}(\textbf{w}) = \varphi_{i}(\textbf{w}^T\textbf{x}_i,y_i)+  \sum_{l=1}^q  g_l(\textbf{w}^{(l)}),
\end{equation}
where $g_l(\textbf{w}^{(l)})$ is the regularization function defined on $\w^{(l)}$. It is easy to find that if $g(\cdot)$ is $L_2$ or $L_1$ norm, $g_l(\cdot)$ is also $L_2$ or $L_1$ norm.

Then we can get the gradient:
\begin{align}\label{eq:gradient}
  \nabla f_i(\w) = \nabla\varphi_i(\w^T\x_i,y_i)\x_i + \sum_{l=1}^q \nabla g_l(\textbf{w}^{(l)}).
\end{align}
We can find that the main computation of the gradient $\nabla f_i(\w)$ is to calculate the inner product $\w^T\x_i$ and $\nabla g_l(\textbf{w}^{(l)})$. Since
\begin{align*}
  \w^T\x_i = \sum_{l=1}^{q}\w^{(l)T}\x_i^{(l)},
\end{align*}
we can distributively complete the computation.

\begin{figure}
  \centering
  \includegraphics[width=0.21\textwidth]{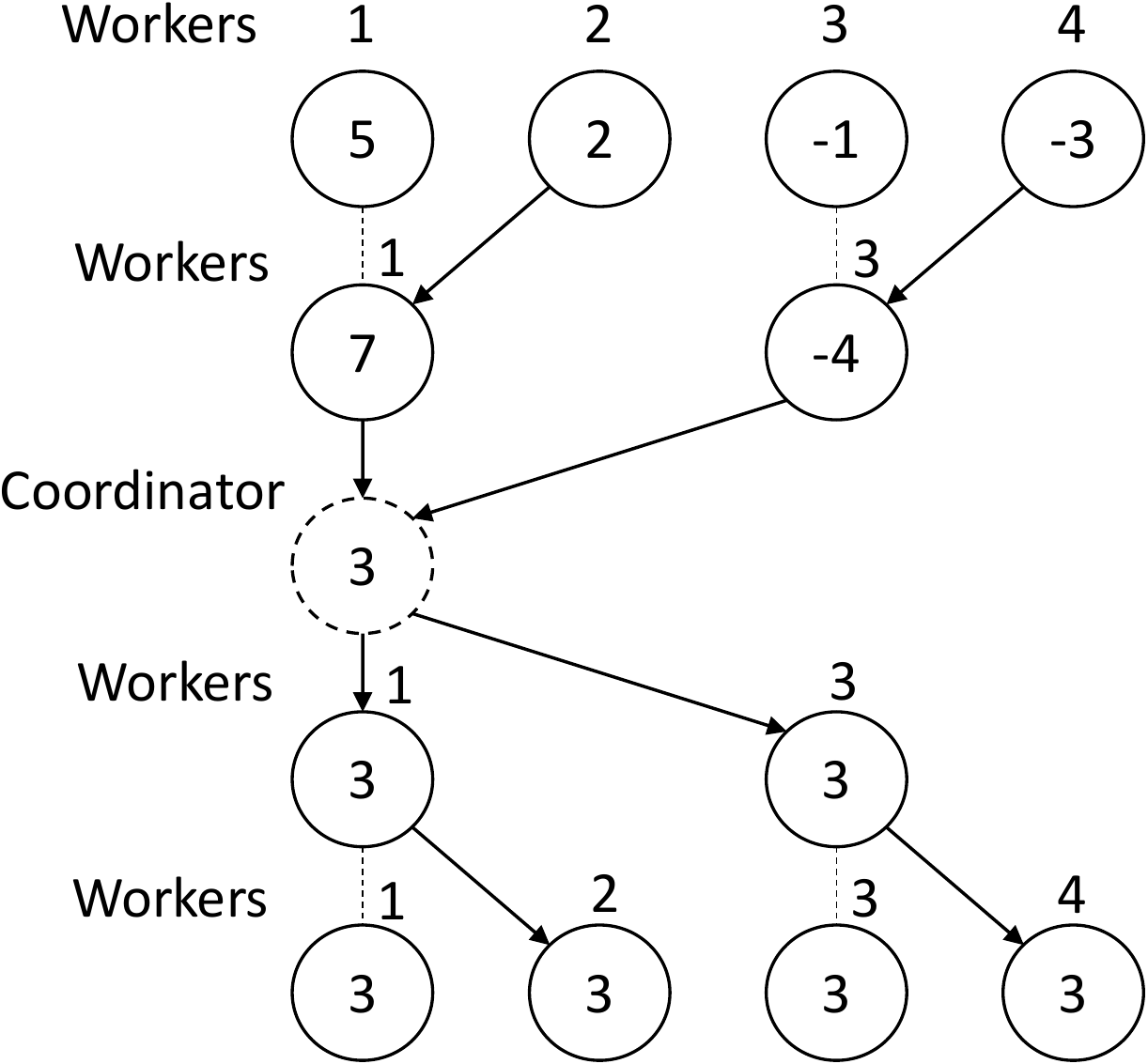}\\
  \caption{An example of tree-structured global sum with 4 Workers.}\label{figuretree}
\end{figure}

The whole learning algorithm of FD-SVRG is shown in Algorithm~\ref{LC-SVRGWorker}. The Coordinator is used to help sum $\w^{(l)T}\x_i^{(l)}$ from all Workers, where $\x_i^{(l)}$ is the $i^{th}$ column of $\D^{(l)}$. Here, we use a tree-structured communication~(reduce) scheme to get the global sum. An example of the tree-structured global sum with 4 Workers is shown in Figure~\ref{figuretree}. We pair the Workers so that while Worker\_1 adds the result from Worker\_2, Worker\_3 can add the result from Worker\_4 simultaneously. After the Coordinator has computed the sum, it broadcasts the sum to all Workers in a reverse-order tree structure. Similar tree-structure can be constructed for more Workers. It is faster than the strategy by which all Workers send the result directly to the Coordinator for sum, especially when the number of Workers is large.
\begin{algorithm}[tb]
  \caption{Feature-Distributed SVRG~(FD-SVRG) on the $l$th Worker}
  \label{LC-SVRGWorker}
  \small
   \begin{algorithmic}[1]
    \STATE Initialize $\eta$, $\textbf{w}_0^{(l)}$;
    \FOR{$t = 0,1,2,\cdots$}
    \STATE Compute $\textbf{w}_t^{(l)T}\textbf{D}^{(l)}$;
    \STATE Use tree-structured communication scheme to obtain: $\textbf{w}_t^{T}\textbf{D} = \sum_{l=1}^q \textbf{w}_t^{(l)T}\textbf{D}^{(l)}$;
    \STATE Compute the full gradient: $\textbf{z}^{(l)} = \frac{1}{N}\sum_{i=1}^N \nabla \varphi_{i}(\textbf{w}_{t}^{T}\textbf{x}_{i},y_i)\textbf{x}_{i}^{(l)}$;
   \STATE Let $\tilde{\textbf{w}}_0^{(l)} = \textbf{w}_t^{(l)}$;
   \FOR{$m = 0,1,2,\cdots, M-1$}
    \STATE Pick up an instance $\textbf{x}_{i_m}^{(l)}$ from the local data $\D^{(l)}$ with index $i_m$;
    \STATE Compute $\tilde{\textbf{w}}_{m}^{(l)T}\textbf{x}_{i_m}^{(l)}$;
    \STATE Use tree-structured communication scheme to obtain: $\tilde{\textbf{w}}_{m}^{T}\textbf{x}_{i_m} = \sum_{l=1}^q \tilde{\textbf{w}}_{m}^{(l)T}\textbf{x}_{i_m}^{(l)}$;
     \STATE  $\tilde{\textbf{w}}_{m+1}^{(l)} = \tilde{\textbf{w}}_{m}^{(l)}- \eta(\nabla \varphi_{i_m}(\tilde{\textbf{w}}_{m}^{T}\textbf{x}_{i_m},y_{i_m})\textbf{x}_{i_m}^{(l)} - \nabla \varphi_{i_m}(\tilde{\textbf{w}}_{0}^{T}\textbf{x}_{i_m},y_{i_m})\textbf{x}_{i_m}^{(l)} + \textbf{z}^{(l)} + \nabla g_l(\tilde{\textbf{w}}_m^{(l)}))$;\label{codeupdate}
      \ENDFOR
     \STATE Set $\textbf{w}_{t+1}^{(l)} = \tilde{\textbf{w}}_{M}^{(l)}$;
    \ENDFOR
  \end{algorithmic}
\end{algorithm}

When computing the full gradient, the inner products of all the data are computed only one time for each outer iteration $t$ because the parameter $\w_t$ is constant for each outer iteration $t$. In line 11 of Algorithm~\ref{LC-SVRGWorker}, only $\tilde{\textbf{w}}_{m}^{T}\textbf{x}_{i_m}$ need to be received from Coordinator. The Worker doesn't need to receive $\tilde{\textbf{w}}_{0}^{T}\textbf{x}_{i_m}$ again for all $M$ inner iterations, because the Worker has received $\textbf{w}_t^{T}\textbf{D}=\tilde{\w}^T_0 \D$ when computing the full gradient for each outer iteration $t$.


\subsection{Convergence Analysis}
It is easy to find that the update rule of FD-SVRG is exactly equivalent to that of the non-distributed~(serial) SVRG~\cite{DBLP:conf/nips/Johnson013}. Hence, the convergence property of FD-SVRG is the same as that of the non-distributed SVRG. Please note that the non-distributed SVRG has two options to get $\w_{t+1}$ in the original paper~\cite{DBLP:conf/nips/Johnson013}~(the readers can also refer to Algorithm 2 in Appendix A of the supplementary material). The authors of~\cite{DBLP:conf/nips/Johnson013} have only proved the convergence of Option II without proving the convergence of Option I. But in FD-SVRG, we prefer to choose Option I because we need to make the parameter identical for different machines with feature partitioned. If Option II is taken, there exists extra communication for the random value. In this section, we prove that Option I is also convergent with a linear convergence rate.

\begin{theorem}\label{theorem:convergence}
  Assume $f(\w)$ is $\mu$-strongly convex and each $f_i(\w)$ is $L$-smooth, which means that $\forall \u, \v$,
  \begin{align*}
    & f_i(\v) \leq f_i(\u)+\nabla f_i(\v)^T(\u-\v) + \frac{L}{2}\|\u-\v\|^2, \\
    & f(\v) \geq f(\u)+\nabla f(\v)^T(\u-\v) + \frac{\mu}{2}\|\u-\v\|^2 .
  \end{align*}
  In the $t^{th}$ outer loop, the inner loop starts with $\tilde{\w}_0$, then we have
  \begin{align*}
  \E\|\tilde{\w}_{M} - \w^*\|^2\leq (a^{M} + \frac{b}{1-a})\|\tilde{\w}_{0} - \w^*\|^2,
  \end{align*}
  where $\eta$ is the step size~(learning rate), $a= 1-\mu\eta + 2L^2\eta^2$, $b = 2L^2\eta^2$, $\w^*$ is the optimal value of $f(\w)$.
\end{theorem}

\begin{proof}
Let $\g_m = \nabla f_{i_m}(\tilde{\w}_m) - \nabla f_{i_m}(\tilde{\w}_0) + \z$, where $\z = \nabla f(\tilde{\w}_0)$ is the full gradient. Then we get $\E[\g_m|\tilde{\w}_m] = \nabla f(\tilde{\w}_m)$. According to the update rule, we have
\begin{align*}
& \|\tilde{\w}_{m+1} - \w^*\|^2 \\
= &\|\tilde{\w}_m - \w^* - \eta \g_m\|^2 \\
= &\|\tilde{\w}_m - \w^*\|^2 - 2\eta \g_m^T(\tilde{\w}_m - \w^*) + \eta^2\|\g_m\|^2.
\end{align*}
According to the definition of $\g_m$, we obtain
\begin{align*}
& \E[\|\g_m\|^2|\tilde{\w}_m] \\
= & \frac{1}{N}\sum_{i=1}^{N} \|\nabla f_{i}(\tilde{\w}_m) - \nabla f_{i}(\tilde{\w}_0) + \z\|^2 \\
\leq & \frac{2}{N}\sum_{i=1}^{N} \|\nabla f_{i}(\tilde{\w}_m) - \nabla f_{i}(\w^*)\|^2 \\
     & + \frac{2}{N}\sum_{i=1}^{N} \|\nabla f_{i}(\w^*) - \nabla f_{i}(\tilde{\w}_0) + \z\|^2 \\
\leq & \frac{2}{N}\sum_{i=1}^{N} \|\nabla f_{i}(\tilde{\w}_m) - \nabla f_{i}(\w^*)\|^2 \\
     & + \frac{2}{N}\sum_{i=1}^{N} \|\nabla f_{i}(\w^*) - \nabla f_{i}(\tilde{\w}_0)\|^2 \\
\leq & 2L^2\|\tilde{\w}_m - \w^*\|^2 + 2L^2\|\tilde{\w}_0 - \w^*\|^2.
\end{align*}
The second inequality uses the fact that $\nabla f(\w^*) = 0$ and $\E \|\xi - \E[\xi]\|^2 \leq \E \|\xi\|^2$. The last inequality uses the smooth property of $f_i(\w)$. Then we have
\begin{align*}
& \E[\|\tilde{\w}_{m+1} - \w^*\|^2] \\
= & \|\tilde{\w}_m - \w^*\|^2 - 2\eta \E[\g_m|\tilde{\w}_m]^T(\tilde{\w}_m - \w^*) \\
  & + \eta^2\E[\|\g_m\|^2|\tilde{\w}_m] \\
\leq & \|\tilde{\w}_m - \w^*\|^2 - \2\eta\nabla f(\tilde{\w}_m)(\tilde{\w}_m-\w^*) \\
     & + 2\eta^2L^2(\|\tilde{\w}_m - \w^*\|^2 + \|\tilde{\w}_0 - \w^*\|^2).
\end{align*}
Using the strongly convex property of $f(\w)$, we obtain
\begin{align*}
& \E[\|\tilde{\w}_{m+1} - \w^*\|^2\|\tilde{\w}_m] \\
\leq & (1-\mu\eta)\|\tilde{\w}_m - \w^*\|^2 - 2\eta(f(\tilde{\w}_m)-f(\w^*)) \\
     & + 2\eta^2L^2(\|\tilde{\w}_m - \w^*\|^2 + \|\tilde{\w}_0 - \w^*\|^2) \\
\leq & (1-\mu\eta + 2L^2\eta^2)\|\tilde{\w}_m - \w^*\|^2 + 2L^2\eta^2\|\tilde{\w}_0 - \w^*\|^2.
\end{align*}
For convenience, let $a= 1-\mu\eta + 2L^2\eta^2$, $b = 2L^2\eta^2$. Taking expectation on the above equation, we obtain
\begin{align*}
& \E\|\tilde{\w}_{m+1} - \w^*\|^2 \\
\leq & a \E\|\tilde{\w}_{m} - \w^*\|^2 + b\|\tilde{\w}_0 - \w^*\|^2 \\
\leq &(a^{m+1} + \frac{b}{1-a})\|\tilde{\w}_{0} - \w^*\|^2.
\end{align*}
Let $m = M-1$, we obtain the result in Theorem~\ref{theorem:convergence}.
\end{proof}

Please note that for the $t^{th}$ outer loop, $\tilde{\w}_0$ actually denotes $\w_t$, and $\tilde{\w}_M$ is actually $\w_{t+1}$. Based on Theorem~\ref{theorem:convergence}, we have
 \begin{align*}
  \E\|\w_t - \w^*\|^2\leq (a^{M} + \frac{b}{1-a})^t\|\w_{0} - \w^*\|^2.
  \end{align*}
When the step size $\eta$ is small enough, $a^M + \frac{b}{1-a} < 1$. It means that our FD-SVRG has a linear convergence rate.

\subsection{Implementation Details}

\subsubsection{Mini-Batch}
FD-SVRG can take a mini-batch strategy as described in~\cite{DBLP:conf/nips/ZhaoYWAL14}. In each iteration, Workers sample a batch of data with batch size $u$. Then, $u$ inner products are computed once and the $u$ scalars are communicated together. Taking a mini-batch cannot reduce the total communication of FD-SVRG, but it can reduce the communication frequency~(times).

\subsection{Complexity Analysis}\label{sec:complexity}

For an iteration of the outer loop in SVRG, there will be $N+M$ gradients to be computed. Here, $M$ is set as the number of local data instances in general. For the instance-distributed methods, data is partitioned by $q$ machines, and each machine has $\frac{N}{q}$ instances. DSVRG sets $M = \frac{N}{q}$. Then each machine will compute $\frac{N}{q}$ gradients in inner loops. There is only one machine at work in inner loops for DSVRG. So it computes $N(1+\frac{1}{q})$ gradients in total during an iteration. When computing full gradient, the center of DSVRG sends parameter to each machine, then receives gradients from each machine. The communication cost is $2qd$. In inner loops, center sends full gradient to machine\_$l$ which is at work. The machine\_$l$ iteratively updates parameter and then returns parameter to center after the iterative updating. The communication is $2d$. So the total communication cost is $2qd+2d$. That means DSVRG computes $N$ gradients with communication cost of $\frac{2qd+2d}{1+\frac{1}{q}} = 2qd$.

For FD-SVRG, when computing one gradient, the communication cost of a tree is $2q$. For the example in Figure~\ref{figuretree}, there are 4 Workers and the communication cost is 8 scalars~(solid arrow). So the total communication cost is $2qN$ for computing $N$ gradients. Compared to the $2qd$ of DSVRG, we can find that when $d>N$, FD-SVRG has lower communication cost. Note that DSVRG only parallels the SVRG algorithm in computing full gradients. Our method parallels the SVRG algorithm both in computing full gradients and in the inner loops. It means that to compute the same number of gradients, the average time of FD-SVRG is lower than that of DSVRG. Furthermore, FD-SVRG also parallels the communication by the tree-structured communication strategy, which can reduce communication time. Hence, FD-SVRG is expected to be much faster than DSVRG, which will be verified in our experiments.

SynSVRG and AsySVRG can be implemented with the Parameter Server framework~(refer to Appendix B of the supplementary material). They need to send many times of vectors in the inner loops of SVRG. The communication cost of them is $O(N+d)$, which is much higher than those of DSVRG and FD-SVRG.

\section{Experiments}

In this section, we choose logistic regression~(LR) with a $L_2$-norm regularization term to conduct experiments. Hence, the formula~(\ref{eq:objectiveF}) is defined as follows:
\begin{equation}\label{equationminlog}
  \min_\textbf{w}  \frac{1}{N} \sum_{i=1}^N \left[\log (1 + e^{-y_i\textbf{w}^T\textbf{x}_i}) + \frac{\lambda}{2}||\textbf{w}||^2\right].
\end{equation}
All the experiments are performed on a cluster of several machines~(nodes) connected by 10GB Ethernet. Each machine
has 12 Intel Xeon E5-2620 cores with 96GB memory.

\subsection{Data Sets}
We use four data sets for evaluation. They are news20, url, webspam and kdd2010. All of these data sets can be downloaded from the LibSVM website\footnote{\url{https://www.csie.ntu.edu.tw/~cjlin/libsvmtools/datasets/}}. The detailed information about these data sets is summarized in Table~\ref{tabledata}. We use 8 Workers to train news20 because it is relatively smaller. For other data sets, 16 Workers are used for training.
\begin{table}
\small
  \centering
   \caption{Data sets for evaluation}\label{tabledata}
  \begin{tabular}{|c|c|c|}
  \hline

  \textbf{Data set} & \textbf{Features} ($d$) & \textbf{Instances} ($N$)\\\hline
  news20 & 1,355,191 & 19,954   \\\hline
  url & 3,231,961 & 2,396,130  \\\hline
  webspam & 16,609,143 & 350,000  \\\hline
  kdd2010 & 29,890,095 & 19,264,097  \\\hline
\end{tabular}
\end{table}

\subsection{Experimental Setting}
We chose DSVRG, AsySVRG and SynSVRG as baselines. AsySVRG and SynSVRG are implemented on Parameter Server framework~(refer to Appendix B of the supplementary material). Besides Workers, AsySVRG and SynSVRG need extra machines for Servers. In our experiments, we take 8 Servers for AsySVRG and 4 Servers for SynSVRG. The number of Workers for AsySVRG and SynSVRG is the same as that for FD-SVRG and DSVRG. The $\textbf{w}$ is initialized with $\textbf{0}$. We set the number of inner loops of each method to be the number of training instances on each Worker. The step-size $\eta$ is fixed during training.

\subsection{Efficiency Comparison} \label{sectionefficiencycomparison}

\begin{figure*}[ht]
  \centering
  \subfigure[news20]{ \includegraphics[width=.23\textwidth]{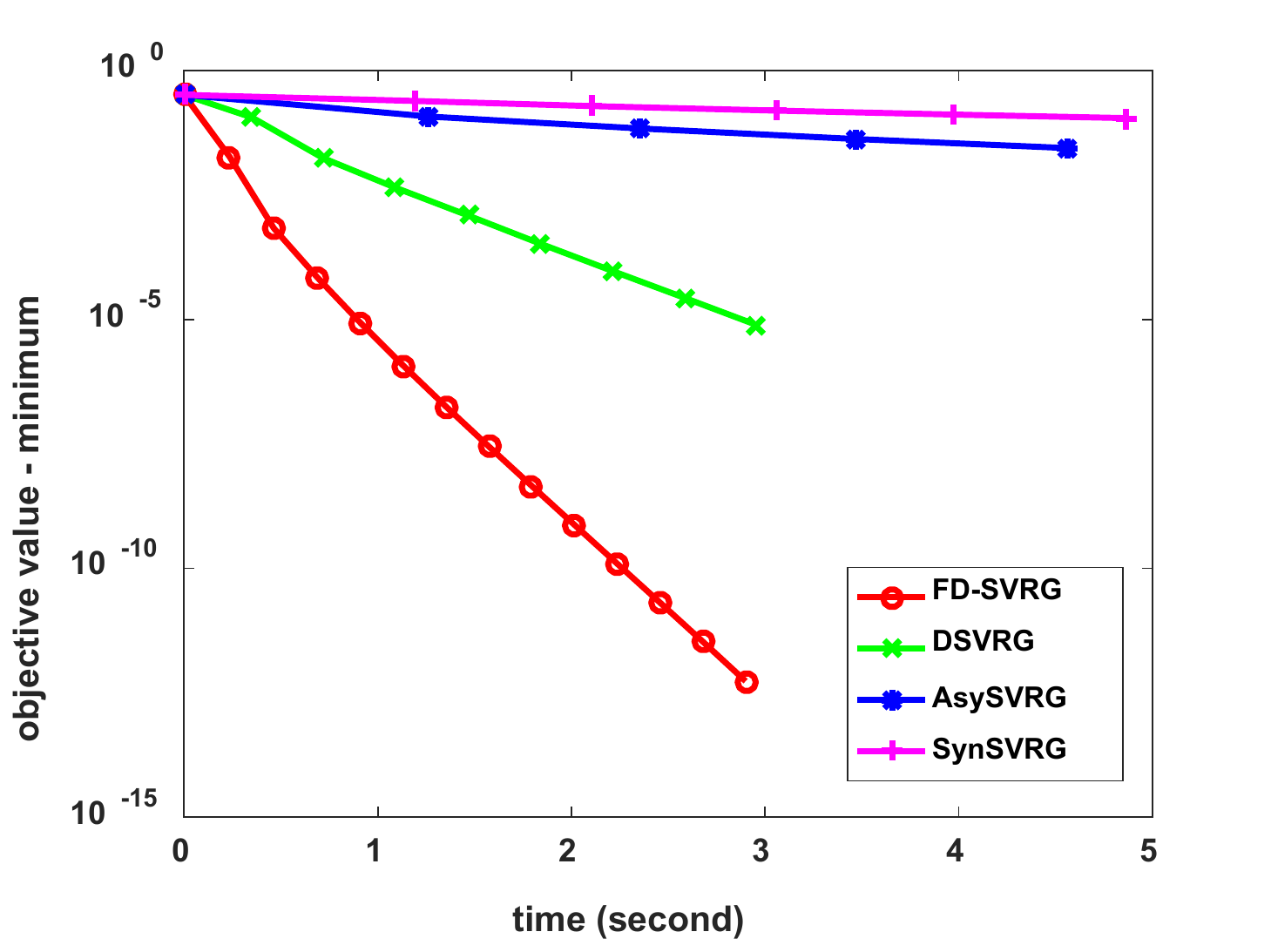}}
  \subfigure[url]{ \includegraphics[width=.23\textwidth]{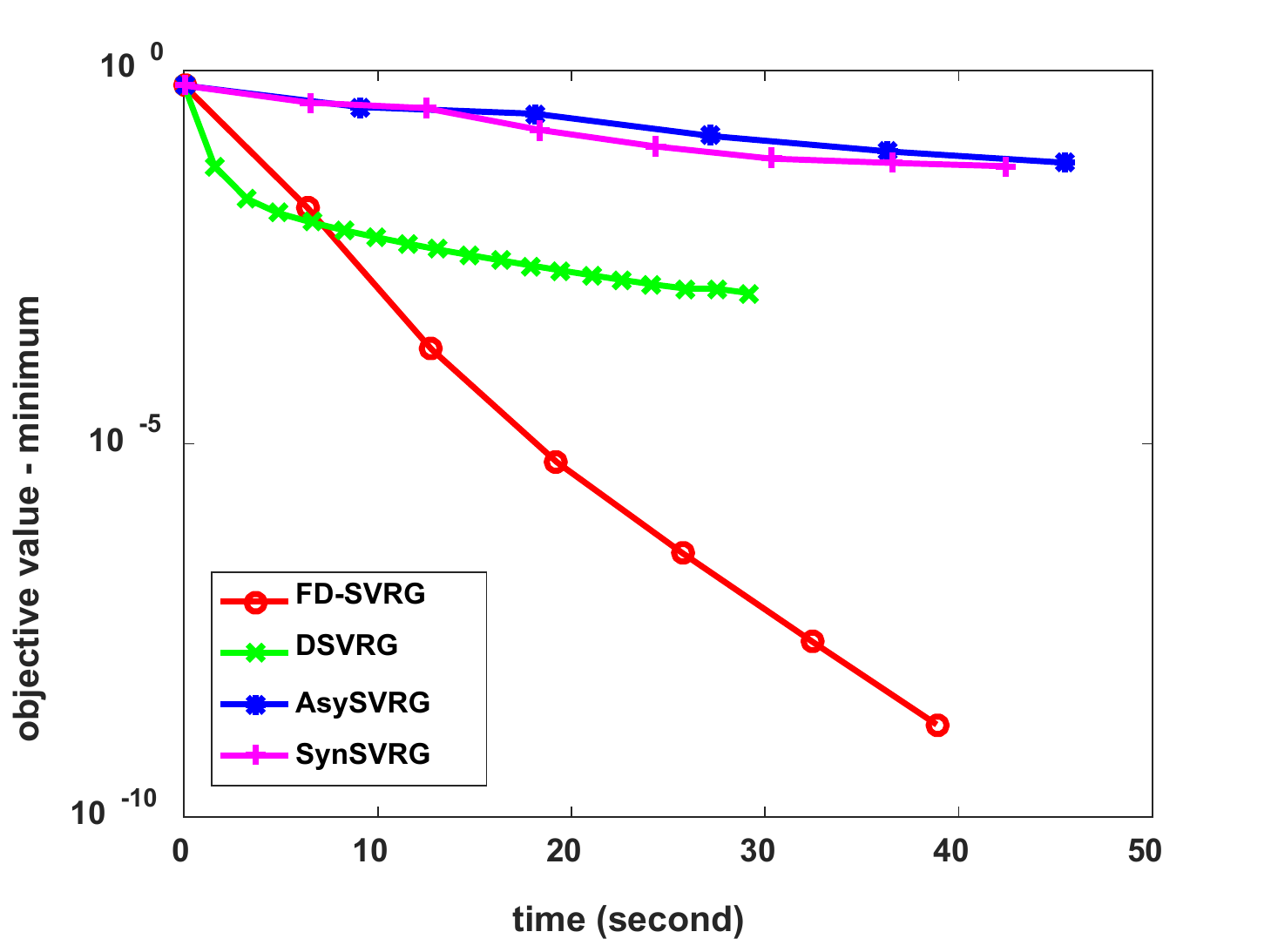}}
  \subfigure[webspam]{ \includegraphics[width=.23\textwidth]{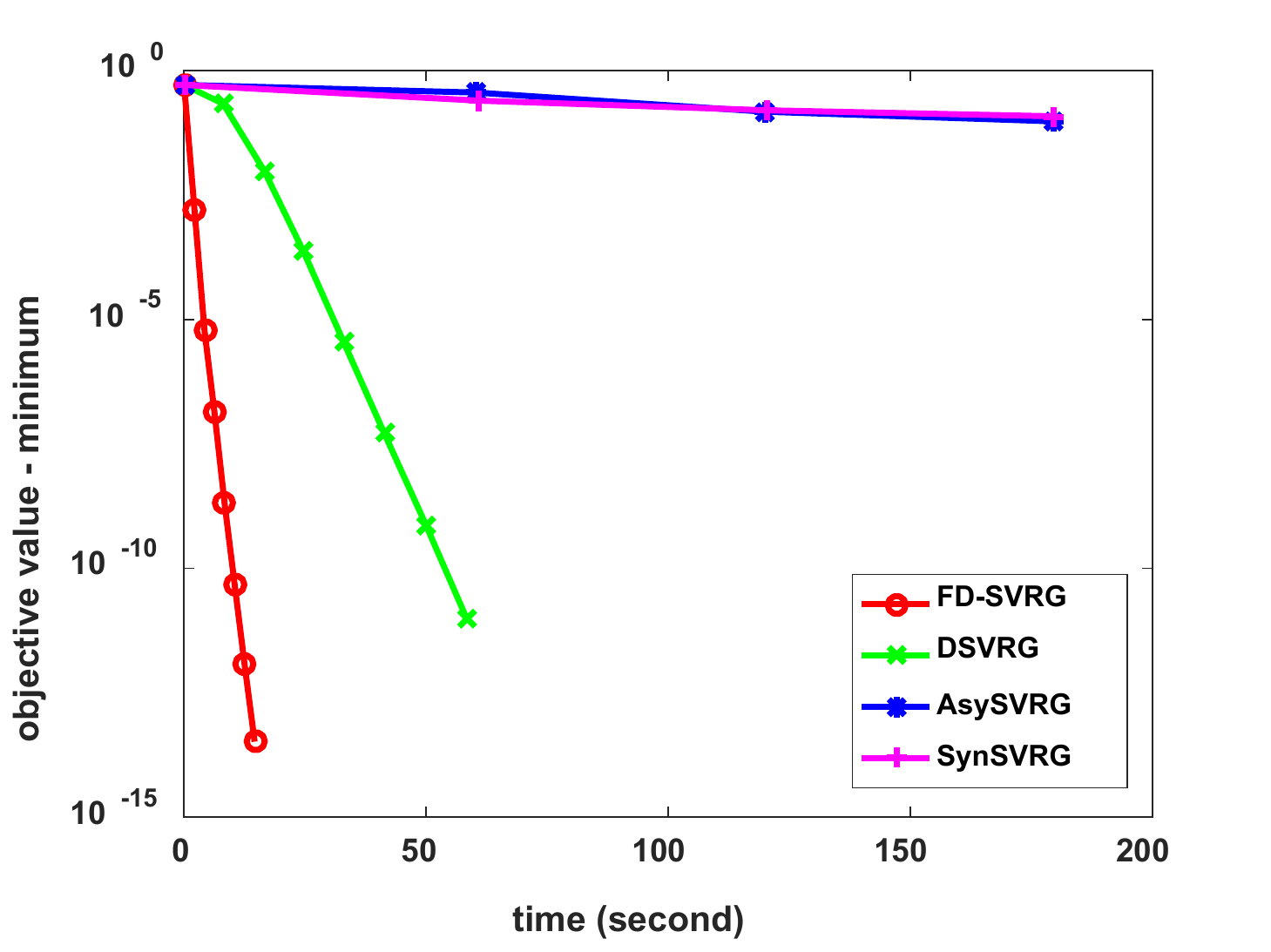}}
  \subfigure[kdd2010]{ \includegraphics[width=.23\textwidth]{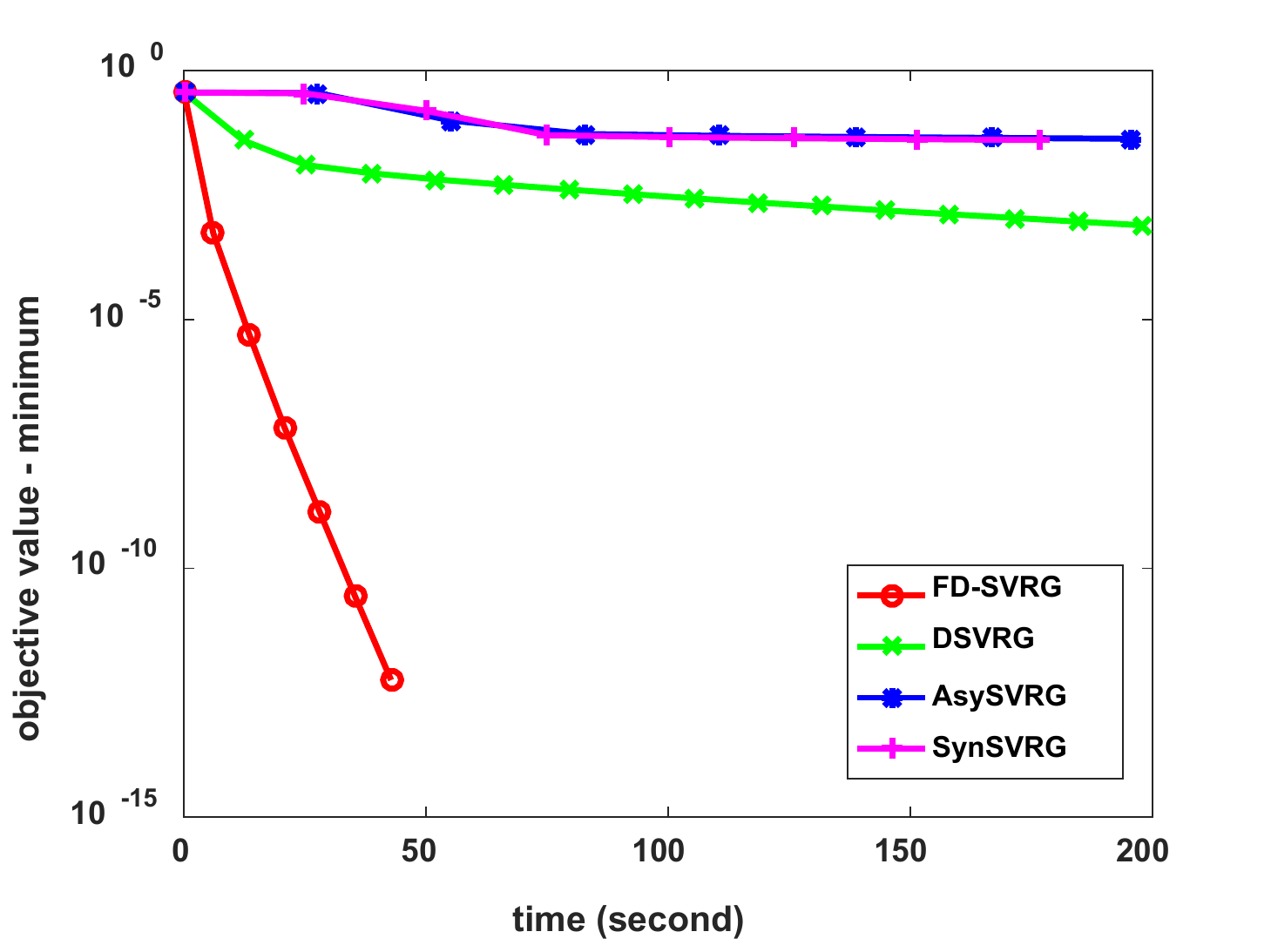}}
  \vspace{-0.3cm}
  \caption{Efficiency comparison in terms of wall-clock time.}\label{efficiencycomparison}
\end{figure*}

\begin{figure*}[ht]
  \centering
  \subfigure[news20]{ \includegraphics[width=.23\textwidth]{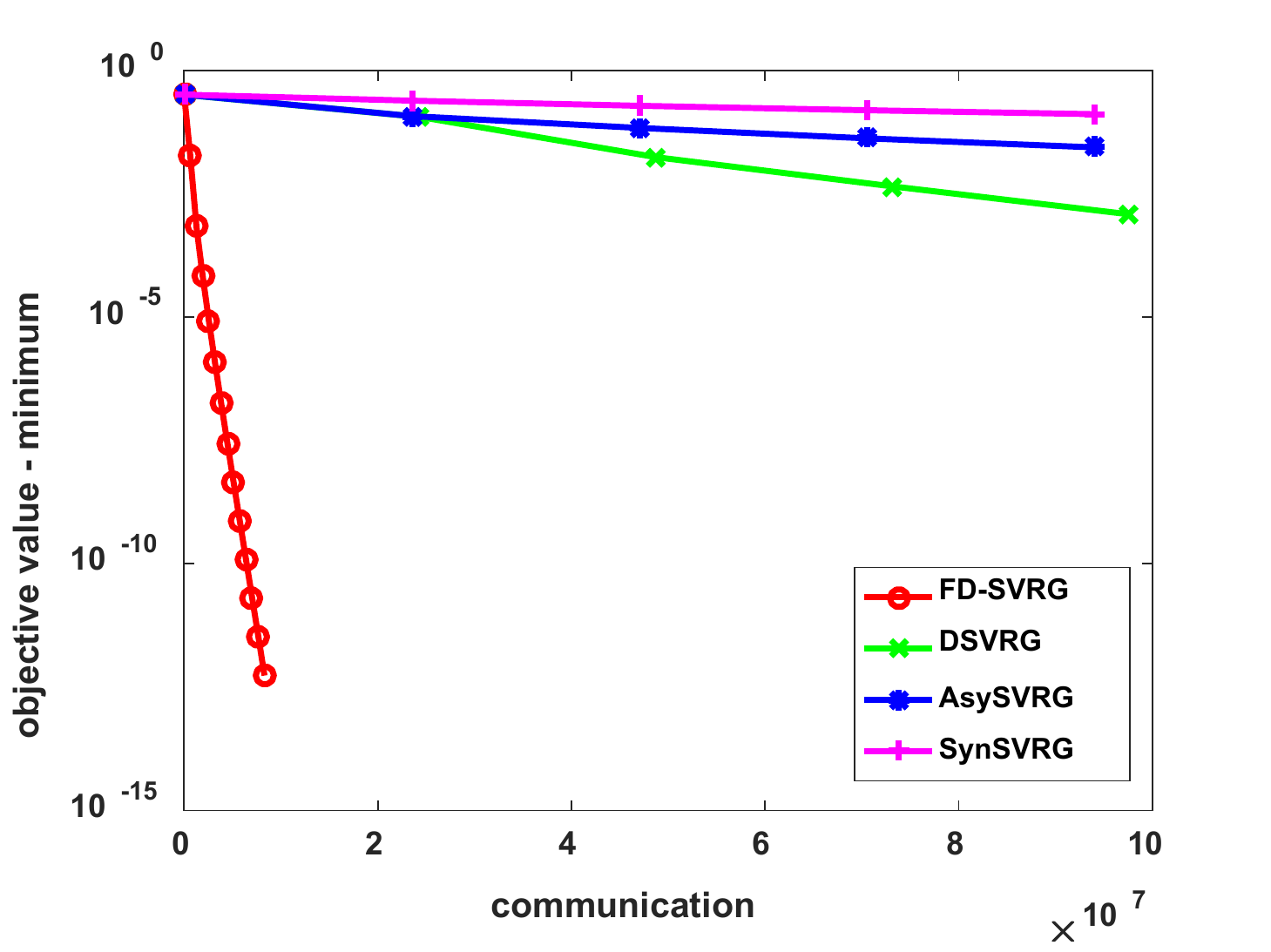}}
  \subfigure[url]{ \includegraphics[width=.23\textwidth]{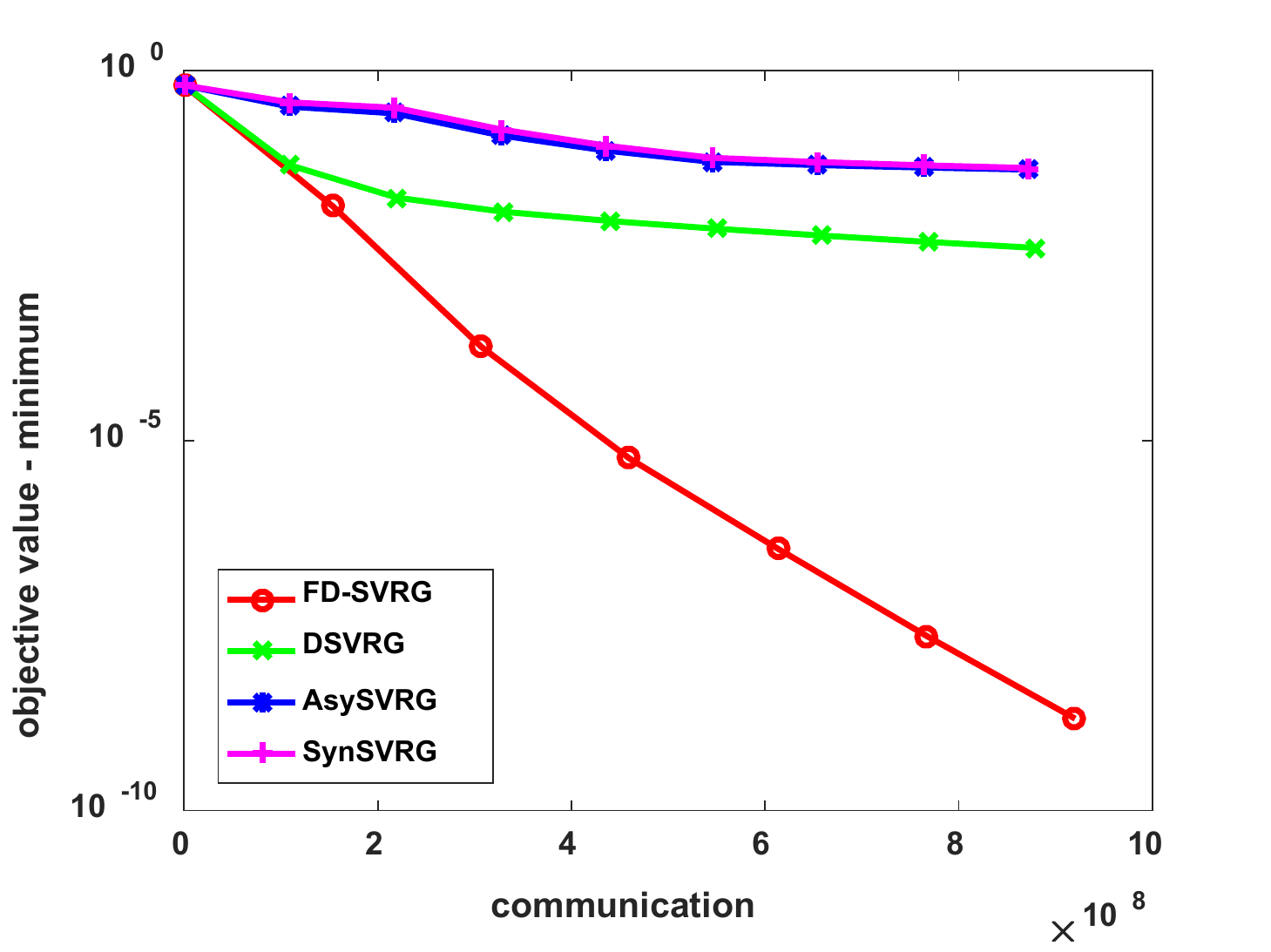}}
  \subfigure[webspam]{ \includegraphics[width=.23\textwidth]{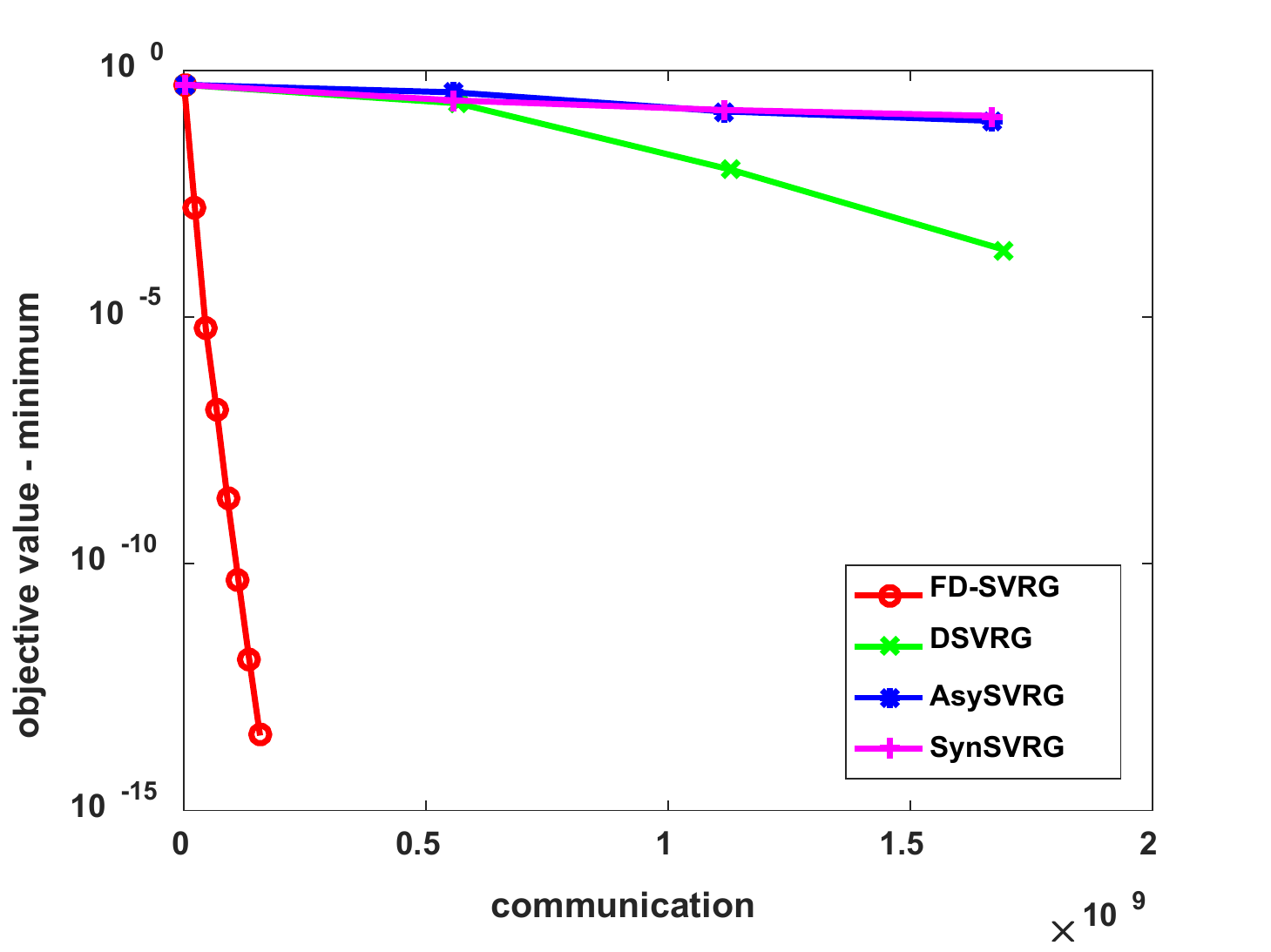}}
   \subfigure[kdd2010]{ \includegraphics[width=.23\textwidth]{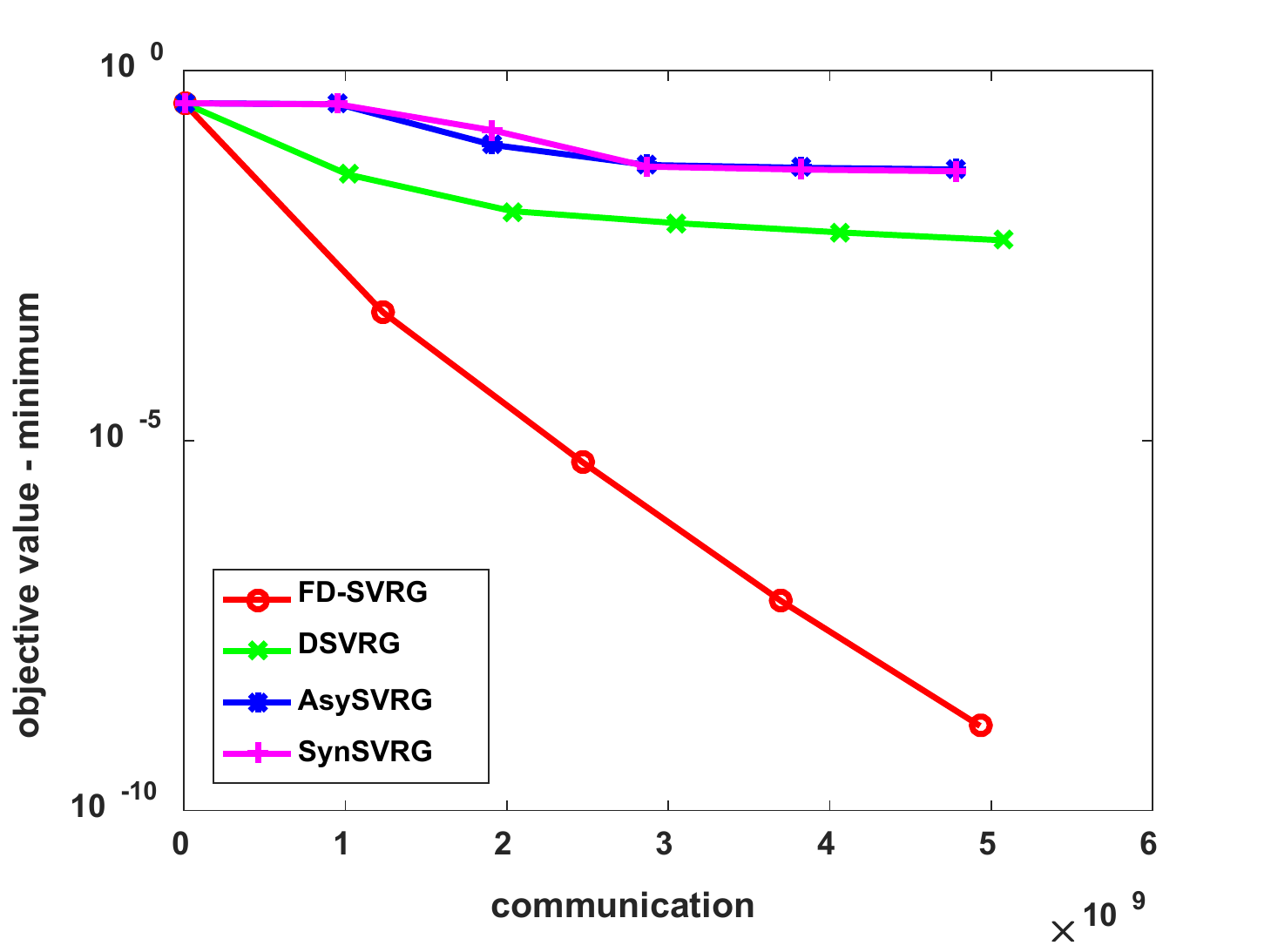}}
  \vspace{-0.3cm}
  \caption{Efficiency comparison in terms of communication cost.}\label{efficiencycommunication}
\end{figure*}

When the regularization hyper-parameter $\lambda$ is $10^{-4}$, the convergence results are shown in Figure~\ref{efficiencycomparison} and Figure~\ref{efficiencycommunication}. In both figures, the vertical axis denotes the gap between the objective function value and the optimal solution value. In Figure~\ref{efficiencycomparison}, the horizontal axis denotes the wall-clock time~(in seconds) used by different methods. In Figure~\ref{efficiencycommunication}, the horizontal axis is the communication cost which denotes how many scalars has been communicated. A $d$ dimensional vector is considered to be $d$ scalars in communication cost. We can find that FD-SVRG achieves the best performance, in terms of both wall-clock time and communication cost.


Because DSVRG is the fastest baseline, we choose DSVRG as a baseline to observe our method's speedup compared to DSVRG. The result is shown in Table~\ref{tablespeedup}, in which the time is recorded when the gap between the objective function value and the optimal value is less than $10^{-4}$. The time of the two methods on the four data sets is on the upper half of Table \ref{tablespeedup}. The lower half part of Table \ref{tablespeedup} is our method's speedup compared to DSVRG. We can find that our method is several times faster than DSVRG.
\begin{table}
\small
  \centering
   \caption{Speedup to DSVRG}\label{tablespeedup}
  \begin{tabular}{|c|c|c|c|}
  \hline
  \multicolumn{2}{|c|}{} & \textbf{DSVRG} & \textbf{FD-SVRG}\\\hline
     & news20  & 2.83  & 0.68 \\\cline{2-4}
   \textbf{Time}  & url & 119.1 &  19.24  \\\cline{2-4}
   \textbf{(in second)}        & webspam & 33.01 & 4.23  \\\cline{2-4}
   & kdd2010 & 400.35 & 13.39  \\\hline
   & news20 & 1 & \textbf{4.16}  \\\cline{2-4}
   \textbf{ Speedup}        & url & 1 & \textbf{6.19}   \\\cline{2-4}
           & webspam & 1 & \textbf{7.8} \\\cline{2-4}
            & kdd2010 & 1 & \textbf{29.9} \\\hline
\end{tabular}
\end{table}



PS-Lite~\cite{DBLP:conf/osdi/LiAPSAJLSS14} has been widely used by both academy and industry. We also compare FD-SVRG with PS-Lite. The original implementation of PS-Lite is based on SGD~\cite{DBLP:conf/osdi/LiAPSAJLSS14}. We denote it as PS-Lite~(SGD). In particular, PS-Lite~(SGD) is an asynchronous SGD implemented based on PS-Lite which is provided by the authors of~\cite{DBLP:conf/osdi/LiAPSAJLSS14}. We summarize the speedup to \mbox{PS-Lite}~(SGD) in Table~\ref{tablesgdspeedup}. The time is recoded when the gap between the objective function value and the optimal value is less than $10^{-4}$. We can find that our method is hundreds even thousands of times faster than PS-Lite~(SGD).

To evaluate the influence of the regularization hyper-parameter $\lambda$, we choose webspam data set to conduct experiments by setting $\lambda = 10^{-3},10^{-5}$. The results are shown in Figure~\ref{lambdacomparison}. Once again, our method achieves the best performance in both cases.

\begin{table}
\small
  \centering
   \caption{Speedup to PS-Lite~(SGD)}\label{tablesgdspeedup}
  \begin{tabular}{|c|c|c|c|}
  \hline
  \multicolumn{2}{|c|}{} & \textbf{PS-Lite (SGD)} & \textbf{FD-SVRG}\\\hline
      & news20  & $>$1000  & 0.69 \\\cline{2-4}
  \textbf{Time}  & url & $>$2000 & 19.25   \\\cline{2-4}
    \textbf{(in second)}       & webspam & 827.12 & 4.23  \\\cline{2-4}
    & kdd2010 & $>$2000 & 13.39  \\\hline
   & news20 & 1 & \textbf{$>$1449}  \\\cline{2-4}
      \textbf{ Speedup}     & url & 1 & \textbf{$>$103}   \\\cline{2-4}
           & webspam & 1 & \textbf{196} \\\cline{2-4}
            & kdd2010 & 1 & \textbf{$>$149} \\\hline
\end{tabular}
\end{table}

\begin{figure}[ht]
  \centering
  \subfigure[$\lambda = 10^{-3}$]{ \includegraphics[width=.23\textwidth]{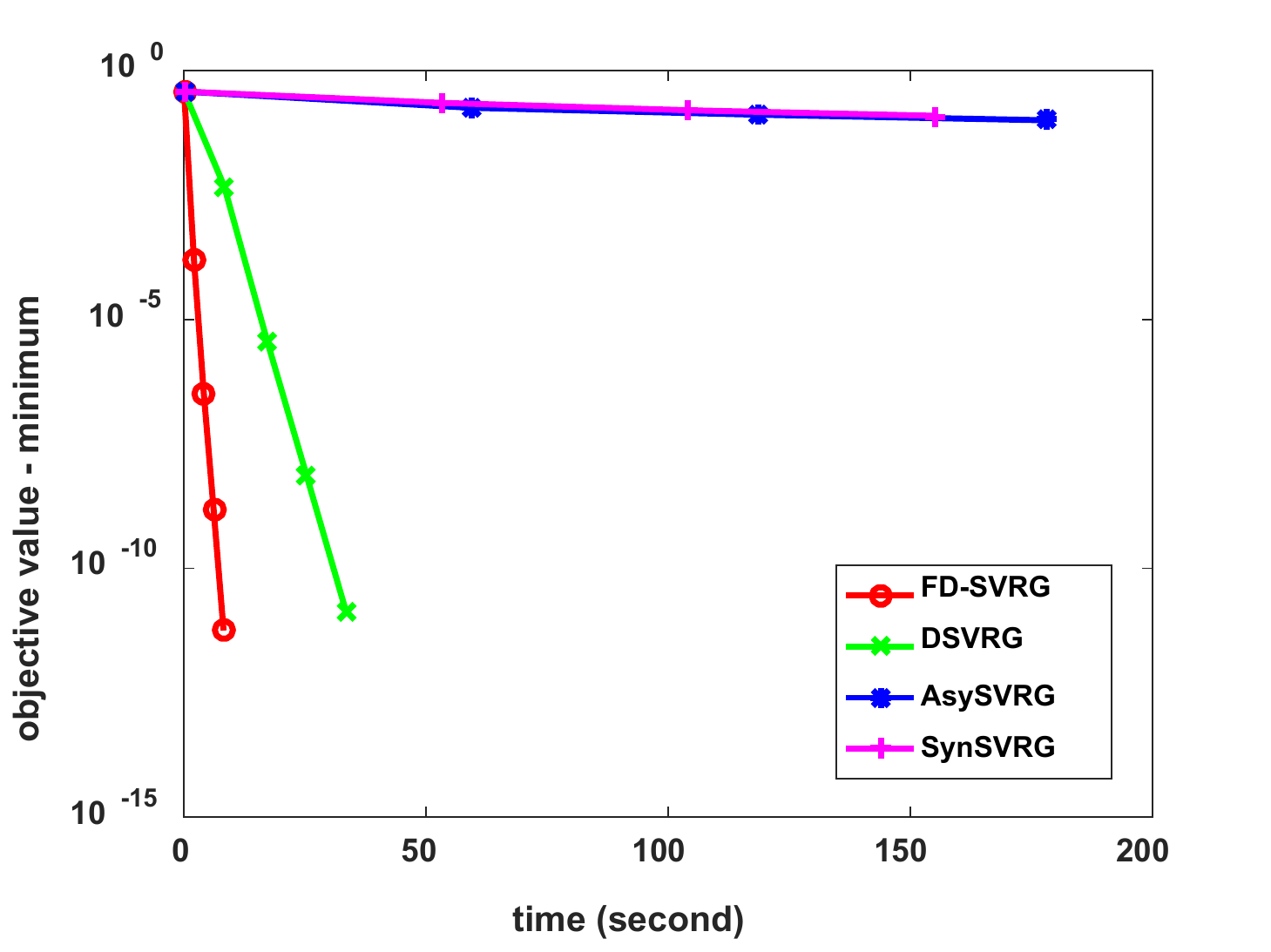}}
  \subfigure[$\lambda = 10^{-5}$]{ \includegraphics[width=.23\textwidth]{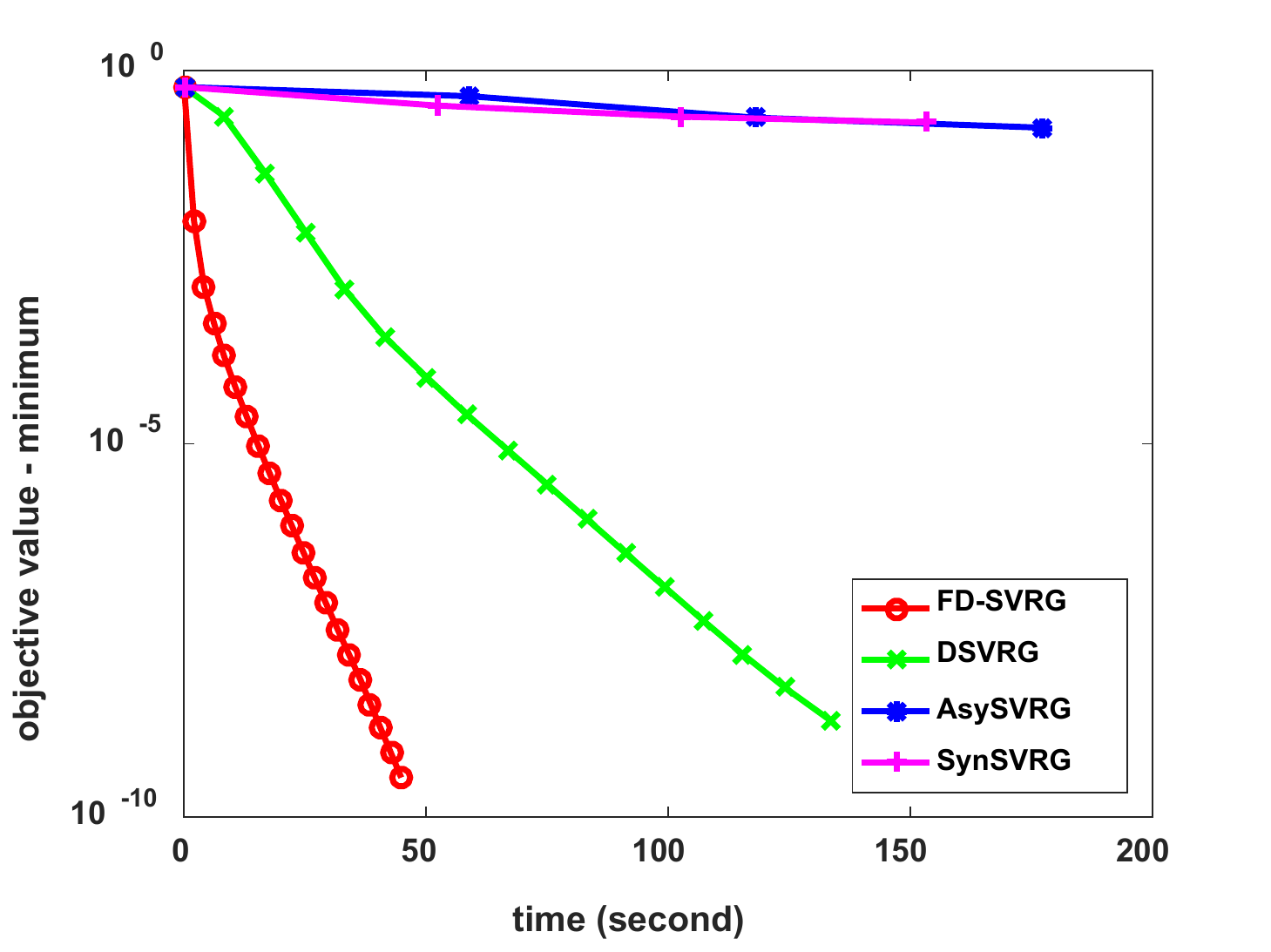}}
  \vspace{-0.3cm}
  \caption{Efficiency comparison in terms of wall-clock time for different $\lambda$.}\label{lambdacomparison}
\end{figure}

\subsection{Scalability}
By changing the number of Workers, we can compute a speedup for FD-SVRG. The speedup is defined as follows:
\begin{equation}
 \nonumber
  speedup = \frac{run\ time\ with\ 1\ Worker}{run\ time\ with\ q\ Workers}
\end{equation}
We take the Worker number to be 1, 4, 8, 16. When the gap between the objective function value and the optimal solution value is less than $10^{-4}$, we stop the training process and record the time.

The speedup result is shown in Figure~\ref{webspeedup}. We can find that FD-SVRG has a speedup close to ideal result. Hence, FD-SVRG has a strong scalability to handle large-scale problems.
\begin{figure}[ht]
  \centering
  \includegraphics[width=.23\textwidth]{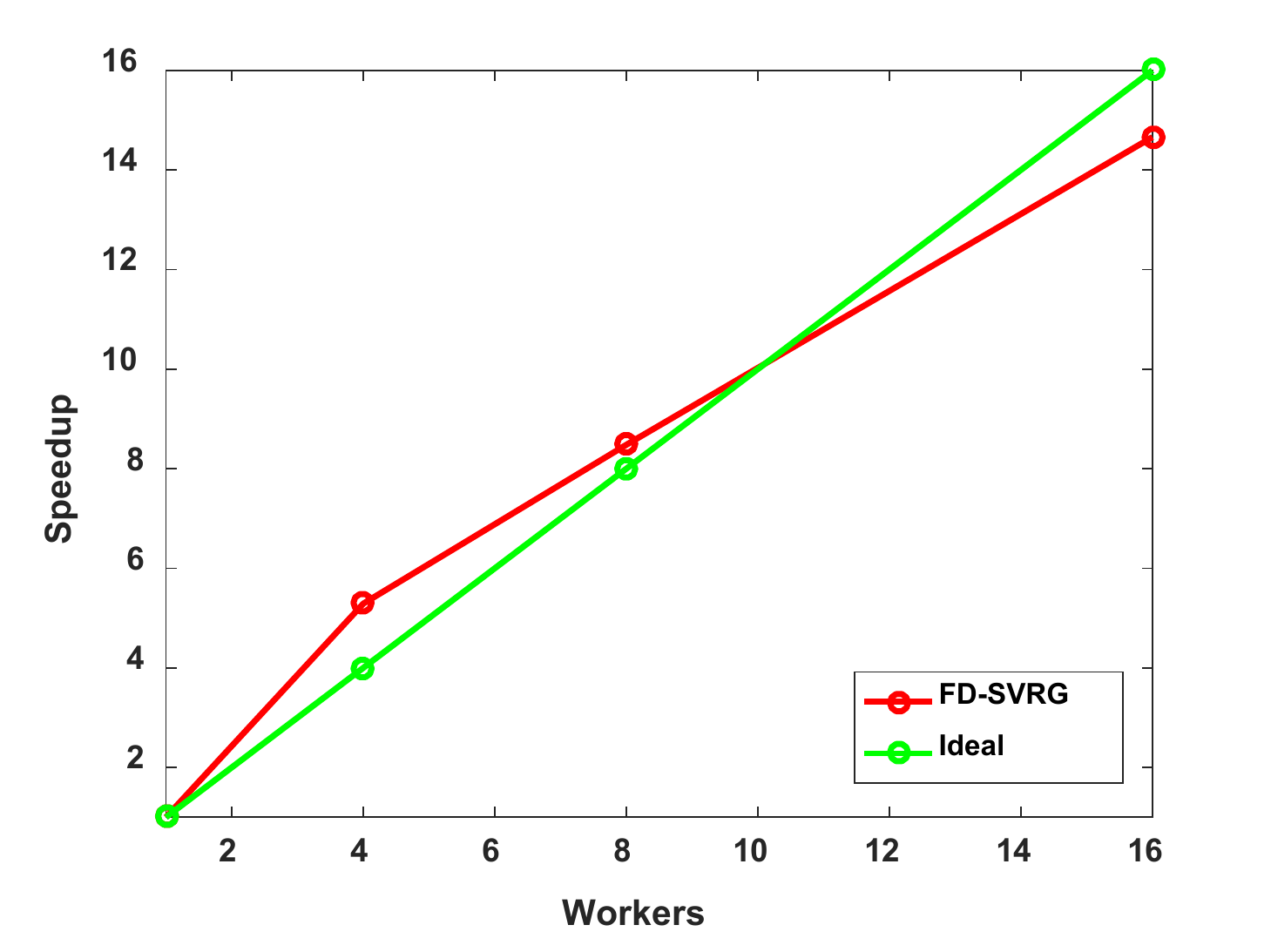}
 \vspace{-0.3cm}
  \caption{Speedup of FD-SVRG on webspam.}\label{webspeedup}
\end{figure}

\section{Conclusion}
Most existing distributed learning methods for linear classification are instance-distributed, which cannot achieve satisfactory results for high-dimensional applications when the dimensionality is larger than the number of instances. In this paper, we propose a novel distributed learning method, called FD-SVRG, by adopting a feature-distributed data partition strategy. Experimental results show that our method can achieve the best performance for cases when the dimensionality is larger than the number of instances.

%
%


\bibliography{bibfile}
\bibliographystyle{icml2018}


\clearpage

\appendix
\section{Serial SVRG}\label{appendix:svrg}
The learning procedure of the non-distributed~(serial) SVRG is shown in Algorithm~\ref{ProcedureSVRG}, where $\eta$ is the learning rate, $\w_t$ denotes the parameter value at iteration $t$, $M$ is a hyper-parameter.

\begin{algorithm}[htb]
  \caption{SVRG}
  \label{ProcedureSVRG}
   \begin{algorithmic}[1]
    \STATE Initialize $\eta$, $\textbf{w}_0$, $M$;
    \FOR{$t = 0,1,2,\cdots$}
    \STATE $\textbf{z} = \frac{1}{N}\sum_{i=1}^N \nabla f_i(\textbf{w}_t)$;
    \STATE $\tilde{\textbf{w}}_0 = \textbf{w}_t$;
   \FOR{$m = 0,1,2,\cdots, M-1$}
    \STATE Randomly pick $i_m \in \{1,\cdots,N\}$
     \STATE  $\tilde{\textbf{w}}_{m+1} = \tilde{\textbf{w}}_{m}- \eta(\nabla f_{i_m}(\tilde{\textbf{w}}_{m})- \nabla f_{i_m}(\tilde{\textbf{w}}_{0}) + \textbf{z})$;
      \ENDFOR
     \STATE \textbf{option I:} Set $\textbf{w}_{t+1}= \tilde{\textbf{w}}_{M}$;
      \STATE \textbf{option II:} Set $\textbf{w}_{t+1} = \tilde{\textbf{w}}_{m} $ for randomly chosen $m \in \{1,2,\cdots, M\}$;
    \ENDFOR
  \end{algorithmic}
\end{algorithm}

\section{Distributed SVRG with Parameter Server}\label{appendix:PSsvrg}

Both SynSVRG and AsySVRG can be implemented on the Parameter Server framework to get distributed versions of SVRG. The SynSVRG is shown in Algorithm~\ref{SynSVRG:server} and Algorithm~\ref{SynSVRG:worker}. Algorithm~\ref{SynSVRG:server} is the operations of Servers and Algorithm~\ref{SynSVRG:worker} is the operations of Workers.
\begin{algorithm}[htb]
 \caption{Task of Server$\_k$ in SynSVRG}
  \label{SynSVRG:server}
  \begin{algorithmic}[1]
   \STATE Initialize $\eta$, $\textbf{w}_0^{(k)}$;
  \FOR{$t = 0,1,2,\cdots$}
    \STATE $\tilde{\textbf{w}}_0^{(k)} = \textbf{w}_t^{(k)}$;
   \STATE Send $\textbf{w}_t^{(k)}$ to all Workers;
   \STATE Receive $\textbf{z}_1^{(k)},\textbf{z}_2^{(k)},\cdots,\textbf{z}_q^{(k)}$ from the $q$ Workers;
    \STATE Compute the full gradient $\textbf{z}^{(k)} = \frac{1}{N}\sum_{l=1}^q \textbf{z}_l^{(k)}$;
    \FOR{$m = 0,1,2,\cdots,M-1$}
      \STATE Send $\tilde{\textbf{w}}_{m}^{(k)}$ to all Workers;
        \STATE Receive $\nabla_{m_1}^{(k)},\nabla_{m_2}^{(k)},\cdots,\nabla_{m_q}^{(k)}$ from $q$ Workers;
       \STATE $\nabla_m^{(k)} = \frac{1}{q}\sum_{l=1}^q \nabla_{m_l}^{(k)}$;
        \STATE  $\tilde{\textbf{w}}_{m+1}^{(k)} = \tilde{\textbf{w}}_{m}^{(k)}- \eta(\nabla_m^{(k)} + \textbf{z}^{(k)})$;
    \ENDFOR
    \STATE $\textbf{w}^{(k)}_{t+1} = \tilde{\textbf{w}}_{M}^{(k)}$;
  \ENDFOR
  \end{algorithmic}
\end{algorithm}

\begin{algorithm}[htb]
  \caption{Task of Worker$\_l$ in SynSVRG}
   \label{SynSVRG:worker}
   \begin{algorithmic}[1]
    \FOR{$t = 0,1,2,\cdots$}
     \STATE Receive $\textbf{w}^{(1)}_t, \textbf{w}^{(2)}_t, \cdots,\textbf{w}^{(p)}_t$ from $p$ Servers and combine the complete $\textbf{w}_t = (\textbf{w}^{(1)}_t,\textbf{w}^{(2)}_t,\cdots,\textbf{w}^{(p)}_t)$;
    \STATE Compute the \emph{local gradient sum}  $\textbf{z}_{l} = \sum_{i\in D_l} \nabla f_i(\textbf{w}_t)$;
    \STATE Cut $\textbf{z}_{l}$ into $p$ parts $\{\textbf{z}_{l}^{(1)},\textbf{z}_{l}^{(2)},\cdots,\textbf{z}_{l}^{(p)}\}$ and send the $p$ parts to $p$ Servers, where $\textbf{z}_{l}^{(k)}$ is sent to Server\_$k$;
   \FOR{$m = 0,1,2,\cdots, M-1$}
    \STATE Receive $\tilde{\textbf{w}}_{m}^{(1)}, \tilde{\textbf{w}}_{m}^{(2)}, \cdots,\tilde{\textbf{w}}_{m}^{(p)}$ from $p$ Servers and combine complete $\tilde{\textbf{w}}_{m} = (\tilde{\textbf{w}}_{m}^{(1)}, \tilde{\textbf{w}}_{m}^{(2)}, \cdots,\tilde{\textbf{w}}_{m}^{(p)})$.
    \STATE Pick up an instance $\textbf{x}_{i_m}$ from the local data $\D_{(l)}$ with index $i_m$;
    \STATE $\nabla_{m_l} = \nabla f_{i_m}(\tilde{\textbf{w}}_{m}) - \nabla f_{i_m}(\textbf{w}_t)$;
        \STATE Cut $\nabla_{m_l}$ into $p$ parts $\{\nabla_{m_l}^{(1)},\nabla_{m_l}^{(2)},\cdots,\nabla_{m_l}^{(p)}\}$ and send the $p$ parts to $p$ Servers, where $\nabla_{m_l}^{(k)}$ is sent to Server\_$k$;
      \ENDFOR
    \ENDFOR
  \end{algorithmic}
\end{algorithm}

The AsySVRG is shown in Algorithm~\ref{AsySVRG:server} and Algorithm~\ref{AsySVRG:worker}. Algorithm~\ref{AsySVRG:server} is the operations of Servers and Algorithm~\ref{AsySVRG:worker} is the operations of Workers.

\begin{algorithm}[htb]
 \caption{Task of Server$\_k$ in AsySVRG}
  \label{AsySVRG:server}
  \begin{algorithmic}[1]
   \STATE Initialize $\eta$, $\textbf{w}_0^{(k)}$;
  \FOR{$t = 0,1,2,\cdots$}
    \STATE $\tilde{\textbf{w}}_0^{(k)} = \textbf{w}_t^{(k)}$;
   \STATE Send $\textbf{w}_t^{(k)}$ to all Workers;
   \STATE Receive $\textbf{z}_1^{(k)},\textbf{z}_2^{(k)},\cdots,\textbf{z}_q^{(k)}$ from the $q$ Workers;
    \STATE Compute the full gradient $\textbf{z}^{(k)} = \frac{1}{N}\sum_{l=1}^q \textbf{z}_l^{(k)}$;
    \STATE $m = 0$;
    \REPEAT
     \IF{Receive \emph{pull} request from Worker$\_l$}
        \STATE Send $\tilde{\textbf{w}}_{m}^{(k)}$ to Worker$\_l$;
     \ELSIF{Receive \emph{push} request from Worker$\_l$}
     \STATE  Receive $\nabla_m^{(k)}$ from  Worker$\_l$;
      \STATE  $\tilde{\textbf{w}}_{m+1}^{(k)} = \tilde{\textbf{w}}_{m}^{(k)}- \eta(\nabla_m^{(k)} + \textbf{z}^{(k)})$;
      \STATE $m$ increases by $1$;
     \ENDIF
    \UNTIL{$m > M-1 $}
    \STATE Send \emph{End Signal} to all Workers;
    \STATE $\textbf{w}^{(k)}_{t+1} = \tilde{\textbf{w}}_{M}^{(k)}$;
  \ENDFOR
  \end{algorithmic}
\end{algorithm}

\begin{algorithm}[htb]
  \caption{Task of Worker$\_l$ in AsySVRG}
   \label{AsySVRG:worker}
   \begin{algorithmic}[1]
    \FOR{$t = 0,1,2,\cdots$}
     \STATE Receive $\textbf{w}^{(1)}_t, \textbf{w}^{(2)}_t, \cdots,\textbf{w}^{(p)}_t$ from $p$ Servers and combine the complete $\textbf{w}_t = (\textbf{w}^{(1)}_t,\textbf{w}^{(2)}_t,\cdots,\textbf{w}^{(p)}_t)$;
    \STATE Compute the \emph{local gradient sum}  $\textbf{z}_{l} = \sum_{i\in D_l} \nabla f_i(\textbf{w}_t)$;
    \STATE Cut $\textbf{z}_{l}$ into $p$ parts $\{\textbf{z}_{l}^{(1)},\textbf{z}_{l}^{(2)},\cdots,\textbf{z}_{l}^{(p)}\}$ and send the $p$ parts to $p$ Servers, where $\textbf{z}_{l}^{(k)}$ is sent to Server\_$k$;
  \REPEAT
  \STATE Send \emph{pull} request to all Servers;
    \STATE Receive $\tilde{\textbf{w}}_{m}^{(1)}, \tilde{\textbf{w}}_{m}^{(2)}, \cdots,\tilde{\textbf{w}}_{m}^{(p)}$ from $p$ Servers and combine complete $\tilde{\textbf{w}}_{m} = (\tilde{\textbf{w}}_{m}^{(1)}, \tilde{\textbf{w}}_{m}^{(2)}, \cdots,\tilde{\textbf{w}}_{m}^{(p)})$.
    \STATE Pick up an instance $\textbf{x}_{i_m}$ from the local data $\D_{(l)}$ with index $i_m$;
    \STATE $\nabla_{m} = \nabla f_{i_m}(\tilde{\textbf{w}}_{m}) - \nabla f_{i_m}(\textbf{w}_t)$;
     \STATE Send \emph{push} request to all Servers;
     \STATE Cut $\nabla_{m}$ into $p$ parts $\{\nabla_{m}^{(1)},\nabla_{m}^{(2)},\cdots,\nabla_{m}^{(p)}\}$ and send the $p$ parts to $p$ Servers, where $\nabla_{m}^{(k)}$ is sent to Server\_$k$;
    \UNTIL{receive \emph{End Signal} from Servers};
    \ENDFOR
  \end{algorithmic}
\end{algorithm}

%
%
%

\end{document}